\documentclass[conference]{IEEEtran}
\IEEEoverridecommandlockouts
\usepackage{times}

\usepackage[numbers]{natbib}
\usepackage{multicol}
\usepackage{amsfonts}
\usepackage[bookmarks=true]{hyperref}
\usepackage{stfloats}

\usepackage{amsmath}
\usepackage{amssymb}  
\usepackage{amsthm}
\usepackage{bm}
\usepackage{todonotes}
\usepackage{multicol}
\usepackage[linesnumbered,ruled,noend]{algorithm2e}
\usepackage{xcolor}
\usepackage[most]{tcolorbox}
\usepackage{ifthen}
\usepackage{cuted}

\SetCommentSty{mycommfont}

\makeatletter
\newcommand{\removelatexerror}{\let\@latex@error\@gobble}
\makeatother

\usepackage[capitalise]{acro}

\DeclareAcronym{SLS}{
  short = SLS,
  long  = system level synthesis,
}

\DeclareAcronym{OFC}{
  short = OFC,
  long  = output-feedback control,
}

\DeclareAcronym{NLP}{
  short = NLP,
  long  = nonlinear programming,
}

\DeclareAcronym{NMPC}{
  short = NMPC,
  long  = nonlinear model predictive control,
}
\DeclareAcronym{MPC}{
  short = MPC,
  long  = model predictive control,
}
\DeclareAcronym{RNMPC}{
  short = RNMPC,
  long  = robust nonlinear model predictive control,
}
\DeclareAcronym{ISS}{
  short = ISS,
  long  = input-to-state stability,
}
\DeclareAcronym{RPI}{
  short = RPI,
  long  = robust positively invariant,
}
\DeclareAcronym{QP}{
  short = QP,
  long  = quadratic program,
}
\DeclareAcronym{SOCP}{
  short = SOCP,
  long  = second-order cone program,
}
\DeclareAcronym{SCP}{
  short = SCP,
  long  = sequential convex programming,
}
\DeclareAcronym{LQR}{
  short = LQR,
  long  = linear quadratic regulator,
}
\DeclareAcronym{iLQR}{
  short = iLQR,
  long  = iterative linear quadratic regulator,
}
\DeclareAcronym{ESA}{
  short = ESA,
  long  = European Space Agency,
}
\DeclareAcronym{GPU}{
  short = GPU,
  long  = graphics processing unit,
}
\DeclareAcronym{JIT}{
  short = JIT,
  long  = just-in-time,
}
\DeclareAcronym{LTI}{
  short = LTI,
  long  = linear time-invariant,
}
\DeclareAcronym{RCI}{
  short = RCI,
  long  = robust control invariant,
}
\DeclareAcronym{RL}{
  short = RL,
  long  = reinforcement learning,
}

\DeclareAcronym{RTI}{
  short = RTI,
  long  = real-time iteration,
}

\DeclareAcronym{SQP}{
  short = SQP,
  long  = sequential quadratic programming,
}

\DeclareAcronym{LTV}{
  short = LTV,
  long  = linear time-varying,
}
\DeclareAcronym{LQG}{short = LQG, long = linear–quadratic–Gaussian}

\usepackage{tikz}
\usetikzlibrary{arrows.meta,positioning,fit} 
\usetikzlibrary{calc,shapes.geometric}
\tikzset{>=Latex} 

\newcommand{\Safe}{\mathcal{S}}

\newcommand{\Y}{\mathcal{Y}}
\newcommand{\K}{\mathbf{K}}

\newcommand{\Z}{\mathbf{Z}}

\newcommand{\A}{\mathbf{A}}
\newcommand{\B}{\mathbf{B}}
\newcommand{\C}{\mathbf{C}}
\newcommand{\D}{\mathbf{D}}
\newcommand{\E}{\mathbf{E}}
\newcommand{\F}{\mathbf{F}}
\newcommand{\Iblk}{\mathbf{I}}

\newcommand{\nx}{n_\mathrm{x}}
\newcommand{\nuu}{n_\mathrm{u}}
\newcommand{\ny}{n_\mathrm{y}}

\newcommand{\y}{\mathrm{y}}

\newcommand{\Phixw}{\mathbf{\Phi}^{\mathrm{xw}}}
\newcommand{\Phixe}{\mathbf{\Phi}^{\mathrm{xe}}}
\newcommand{\Phiuw}{\mathbf{\Phi}^{\mathrm{uw}}}
\newcommand{\Phiue}{\mathbf{\Phi}^{\mathrm{ue}}}

\newcommand{\kinTm}{k \in \mathbb{N}_{0:T-1}}
\DeclareMathOperator{\proj}{p}

\newcommand{\M}{\mathbf{M}}

\renewcommand{\Y}{\mathbf{y}}

\newcommand{\R}{\mathbb{R}}

\newcommand{\T}{^\top}

\newcommand{\x}{\mathrm{x}}
\renewcommand{\u}{\mathrm{u}}
\newcommand{\w}{\mathrm{w}}
\newcommand{\e}{\mathrm{e}}

\renewcommand{\P}{\bm{\Phi}}

\newcommand{\Frob}{\mathrm{F}}

\newcommand{\defeq}{:=}

\newtheorem{remark}{Remark}
\newtheorem{assumption}{Assumption}
\newtheorem{proposition}{Proposition}

\begin{document}

\title{\fontsize{21}{24} \selectfont
VISION-SLS: Safe Perception-Based Control from Learned Visual Representations via System Level Synthesis
}

\author{
Antoine P. Leeman$^{\star}$$^{1}$ \and Shuyu Zhan$^{\star}$$^{2}$ \and Melanie N. Zeilinger$^{1}$ \and Glen Chou$^{3}$%
\thanks{$^{\star}$ Equal contribution. $^{1}$Institute for Dynamic Systems and Control, ETH Zürich {\tt\small \{aleeman, mzeilinger\}@ethz.ch}.
$^{23}$Georgia Institute of Technology, Schools of $^{2}$Interactive Computing, $^{3}$Aerospace Engineering and Cybersecurity \& Privacy {\tt\small \{szhan45, chou\}@gatech.edu}.
This work was supported by the European Space Agency OSIP 4000133352 and by an ETH Zurich Doc.Mobility Fellowship.}
}

\maketitle

\begin{abstract}
\looseness-1We propose VISION-SLS, a method for nonlinear output-feedback control from high-resolution RGB images which provides robust constraint satisfaction guarantees under calibrated uncertainty bounds despite partial observability, sensor noise, and nonlinear dynamics. To enable scalability while retaining guarantees, we propose: (i) a learned low-dimensional observation map from pretrained visual features with state-dependent error bounds, and (ii) a causal affine time-varying output-feedback policy optimized via System Level Synthesis (SLS). We develop a scalable, novel solver for the resulting nonconvex program that leverages sequential convex programming coupled with efficient Riccati recursions.
On two simulated visuomotor tasks (a 4D car and a 10D quadrotor) with $\ge 512 \times 512$ pixels and a 59D humanoid task with partial observability, our method
enables safe, information-gathering behavior that reduces uncertainty while guaranteeing constraint satisfaction with empirically-calibrated error bounds.
We also validate our method on hardware, safely controlling a ground vehicle from onboard images, outperforming baselines in safety rate and solve times.
Together, these results show that learned visual abstractions coupled with an efficient solver make SLS-based safe visuomotor output-feedback practical at scale.
The code implementation of our method is available at \href{https://github.com/trustworthyrobotics/VISION-SLS}{https://github.com/trustworthyrobotics/VISION-SLS}.
\end{abstract}

\IEEEpeerreviewmaketitle

\begin{figure*}[!b]
    \centering
    \vspace{-12pt}
    \includegraphics[width=\linewidth]{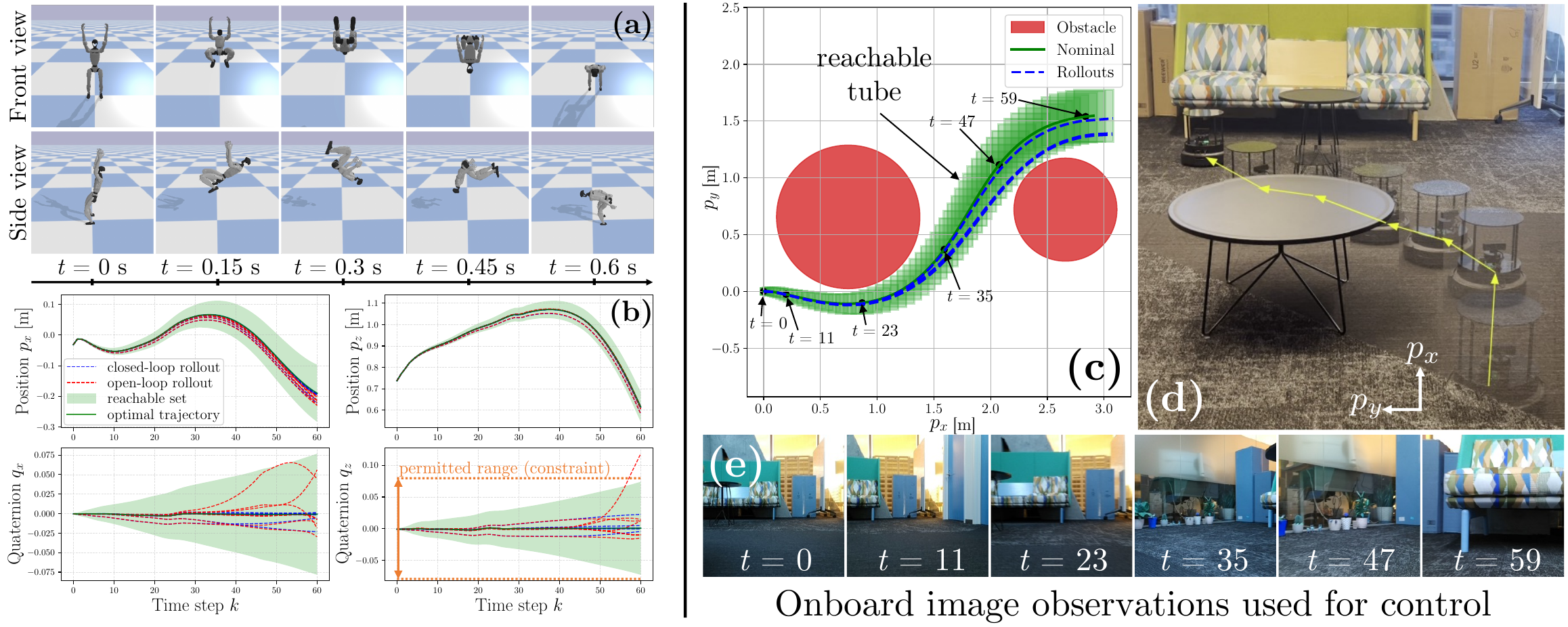}
    \vspace{-20pt}
    \caption{\looseness-1We stabilize a Unitree G1 humanoid (59 states) around a backflip trajectory that is jointly optimized, enabling robust constraint satisfaction. (a) Two views of a timelapse of the backflip trajectory that our method stabilizes. (b) A subset of the reachable tubes (green), closed-loop rollouts (blue), and open-loop rollouts (red). The optimized closed-loop controller via Alg. 1 keeps the system within the reachable tubes, whereas executing the nominal trajectory open-loop (i.e., without feedback) causes the system to exit the tubes and result in constraint violations (e.g., on the quaternion $q_z$). See App. \ref{app:humanoid} for reachable tubes for all states. (c) We evaluate our method on a physical TurtleBot robot, controlled via onboard images. Closed-loop rollouts obtained by using our controller (blue) stay within the tubes (green), avoiding obstacles (red). (d) Rollout time-lapse. (e) Onboard images collected to implement our controller.
    }\vspace{-15pt}
    \label{fig:humanoid_and_hardware}
\end{figure*}

\section{Introduction}

\looseness-1Reliable robot deployment demands safety and performance guarantees. Yet, perception-in-the-loop control, i.e., \textit{output-feedback} control, makes both control synthesis and safety verification difficult: robots must act under partial observability and noisy sensor feedback while efficiently stabilizing their nonlinear dynamics. These challenges are magnified when feedback comes from high-dimensional, non-smooth visual sensor outputs. Due to these difficulties, existing methods typically trade off scalable information-seeking under partial observability against certification of safety and performance. For instance, reinforcement learning (RL) \cite{levine2016end} and imitation learning \cite{chi2025diffusion} scales to high-dimensional visuomotor tasks but yields policies that are difficult to certify. In contrast, model-based methods like model predictive control (MPC) \cite{messerer2022dual,bertsekas2005dynamic}, Lyapunov-based control via sum-of-squares (SOS) \cite{singh2018robust, chou2023synthesizing, chou2022safe} or belief-space planning \cite{DBLP:conf/rss/PlattTKL10} offer guarantees but are mainly limited to linear, low-dimensional settings or rely on the classical separation principle between estimation and control, thereby precluding information-seeking behaviors that are often necessary in robotics for robustness to visual occlusions. 

\subsection{Contributions}

To close this gap, we present a \underline{VIS}ual \underline{O}utput-feedback \underline{N}onlinear \underline{S}ystem \underline{L}evel \underline{S}ynthesis (VISION-SLS), a method for scalable and uncertainty-aware output-feedback control from high-dimensional visual inputs with robust constraint satisfaction guarantees under calibrated uncertainty bounds.
We do not rely on the separation principle: by coupling estimation and control, our policies actively reduce state uncertainty by moving to regions of the state space with more informative sensor readings while maintaining safety guarantees.

To achieve this, our method informs \ac{SLS}-based robust control design with learned visual representations. Our key idea is to represent perception error as bounded disturbance on a reduced observation, so that partial observability, information gathering, and robust constraint satisfaction can be handled directly in the closed-loop response parameterization that is optimized by \ac{SLS}. This yields a scalable output-feedback synthesis pipeline for nonlinear systems with high-dimensional sensor data. Our main contributions are:
\begin{itemize}
    \item \looseness-1We develop a novel controller parameterization for partially observed \emph{nonlinear} systems by coupling nonlinear trajectory optimization with output-feedback \ac{SLS}. In contrast to prior nonlinear \ac{SLS} results, which focus on state feedback \cite{Leeman_2025_TAC,leeman2025guaranteed}, and prior output-feedback \ac{SLS} results \cite{zhou2023safe,DBLP:journals/arc/AndersonDLM19}, which focus on linear or LTV systems, our method enables robust information-gathering control with safety guarantees under nonlinear dynamics (Prop. \ref{prop:robust_constraint_nonlinear}).
    \item We construct a low-dimensional reduced observation from foundation-model features using a dynamics-aware observability loss, and calibrate a state-dependent overbound on the resulting reduction error (Sec.~\ref{sec:method_dino}--\ref{sec:method_error_bound}). This makes high-dimensional visual observations compatible with our robust output-feedback synthesis.
        \item We show that unconstrained output-feedback \ac{LQG} in the \ac{SLS} parameterization has a double-Riccati solution, enabling an efficient solver for the output-feedback subproblems in our method (Prop.~\ref{prop:riccati}). Our solver exploits this structure to scale synthesis to large systems.

    \item We validate the approach on a 4D car with RGB observations, a 10D quadrotor with RGB observations, a 59D humanoid with partial state measurements, and a hardware TurtleBot, showing safe information-gathering behavior, scalability to high-dimensional dynamics, and successful perception-based deployment on hardware.
    \end{itemize}

\section{Related Work}
\label{sec:related_work}
RL has achieved impressive performance in visuomotor control, scaling to high-dimensional tasks and enabling end-to-end pixel-to-torques policies \cite{levine2016end, DBLP:conf/nips/WatterSBR15, DBLP:conf/icra/ShahEKRL21, DBLP:conf/aistats/BanijamaliSGB018, DBLP:conf/icml/ZhangVSA0L19}. However, RL methods are notoriously data-hungry and difficult to certify, as the resulting policies are complex and lack interpretability. While visual foundation models (VFMs) \cite{DBLP:journals/corr/abs-2304-07193, DBLP:journals/corr/abs-2112-05814, zhou2024dino} have recently improved generalization by providing rich feature representations, coupling them with RL still leaves the overall perception–control pipeline without calibrated guarantees.

\Ac{SLS} offers a principled, scalable model-based control framework for large LTV systems for output feedback \cite{DBLP:journals/arc/AndersonDLM19,dean2021guaranteeing,sadraddini2020robust}, with recent extensions to nonlinear state-feedback under robust constraints \cite{Leeman_2025_TAC}. Belief-space planning (BSP) instead casts output-feedback as a partially-observable Markov decision process (POMDP) to model sensor noise and information gathering, but exact solutions are intractable and practical Gaussian/sampling-based approximations (e.g., based on the extended Kalman filter (EKF) or iterative LQG) typically do not consider constraints under uncertain dynamics \cite{DBLP:conf/rss/PlattTKL10,DBLP:conf/l4dc/DeglurkarLTSFT23}.
Finally, a line of work on safe control with learned models and visual inputs has emerged \cite{tong2023enforcingsafetyvisionbasedcontrollers, chou2023synthesizing}, but only heuristically encourages safety or is limited to low-dimensional systems. 

In summary, prior work couples learning with control from images but typically lacks end-to-end guarantees, assumes state reconstruction, or does not handle robust constraints under nonlinear dynamics. 
These limitations motivate our approach, which combines visual abstractions from VFMs with dynamic output-feedback control using \ac{SLS}, enabling scalable, information-gathering, and safe visuomotor control.
\paragraph{Notation} 
\looseness-1For vectors or matrices $a$ and $b$ with the same number of rows, we denote their horizontal concatenation by $[a,~b]$, and their vertical stacking as $(a,b) = [a\T,~b\T]\T$. Let $0_{n\times m}$ and $I_n$ denote the $n\times m$-zeros and $n\times n$ identity matrix, respectively, with dimensions inferred from context when omitted. We use $\nx$, $\nuu$, and $\ny$ for the dimensions of state, input, and observation. The matrix \(\mathbf{Z}\) denotes the block downshift operator, with \(I_{\nx}\) on the first subdiagonal and zeros elsewhere.
For sets, $A \oplus B$ denote the Minkowski sum. 
For a vector $a \in \mathbb{R}^n$, we denote as $a_i$ its $i$th entry. We use the unit $\ell_\infty$ ball $\mathcal{B}^n \!:= \{z\in\mathbb{R}^n : \|z\|_\infty \le 1\}$. We use $\mathbb{N}_{0:T} := \{0,\ldots,T\}$, and use the Frobenius norm as $\|\cdot\|_{\mathrm{F}}$.

\section{Problem Statement}
\label{sec:problem_statement}
We consider discrete-time, partially-observed, nonlinear systems, with states $x \in \mathbb{R}^{\nx}$, controls $u \in \mathbb{R}^{\nuu}$, and observations $y \in \mathbb{R}^{\ny}$,
\begin{subequations}\label{eq:system}
     \begin{align}
     x_{k+1} &= f(x_k,u_k) + E(x_k)w_k\label{eq:dyn} \\
     y_{k+1} &= h(x_k) + F(x_k)e_k,\label{eq:obs}
 \end{align}
 \end{subequations}
where $f: \mathbb{R}^{\nx} \times \mathbb{R}^{\nuu} \to \mathbb{R}^{\nx}$ are the dynamics and $h: \mathbb{R}^{\nx} \rightarrow \mathbb{R}^{\ny}$ are the delayed sensor outputs (IMU, camera images, etc), accounting for processing latency common in vision-based feedback systems.
The functions $E: \mathbb{R}^{\nx} \rightarrow \mathbb{R}^{\nx \times \nx}$, $F: \mathbb{R}^{\nx} \rightarrow \mathbb{R}^{\ny \times \ny}$ represent the effect of the process noise and measurement noise on the system dynamics and observations, respectively. 
\begin{assumption}
\label{ass:calibrated_uncertainty}
We assume bounded process and measurement noise (arising, e.g., from VFM error) with $w_t \in \mathcal{B}^{\nx}$ and $e_t \in \mathcal{B}^{\ny}$, with a system-specific scaling via $E(x_k)$ and $F(x_k)$.
\end{assumption}

\begin{assumption}
We assume the dynamics $f$ are twice continuously differentiable.
\end{assumption}
We aim to guarantee that all reachable states and controls satisfy the constraint 
\begin{equation}
\begin{aligned}
    \hspace{-5pt}\mathcal{S} =  \{(x,u)\in \R^{\nx + \nuu}|~c_i\T(x,u) + b_i \le 0, ~ i = 1, ..., n_\textrm{c}\}\label{eq:constraints}
    \end{aligned}
\end{equation}
despite the process noise and measurement noise.
To ensure robust constraint satisfaction\footnote{For simplicity, we present linear constraints; nonlinear constraints (e.g., obstacle avoidance) are handled via sequential linearizations, similarly in Section \ref{sec:method_scp} and \cite{zhan2025robustly}. Likewise, terminal constraints can also be imposed.} over the prediction horizon, we seek an output-feedback policy
that leverages the history of measurements to compute the control at time $k$. To obtain this policy,
we formulate the following constrained non-convex optimization problem for obtaining an output-feedback policy $\bm{\pi} \defeq (\pi_0(\cdot), \ldots, \pi_{T-1}(\cdot))$:
\begin{subequations}\label{eq:nonlinear_of}
\begin{align}
    \underset{\bm \pi}{\textrm{min}} &\quad J_\pi(\bm \pi) \\
    \textrm{s.t.} &\quad x_{k+1} = f(x_k, u_k) + E(x_k)w_k,\quad \kinTm,\\
        &\quad y_{k+1} = h(x_k) + F(x_k)e_k,\quad \kinTm,\\
        &\quad x_0 \in \mathcal{X}_0, \label{eq:nonlinear_of_constraint_initial_conditions}\\
    &\quad u_k = \pi_k(y_1, \ldots, y_{k}),\quad \forall \kinTm,\label{eq:nonlinear_of_policy} \\
        &\quad (x_k, u_k) \in \Safe, \forall w_k \in \mathcal{B}^{\nx},\forall e_k \in \mathcal{B}^{\ny}, \kinTm,\label{eq:nonlinear_of_constraint}
\end{align}
\end{subequations}
where $J_\pi(\cdot)$ is the objective and $\mathcal{X}_0$ is the set of possible initial states.
The optimization problem \eqref{eq:nonlinear_of} is challenging for three reasons:
\begin{itemize}
    \item The optimization ranges over all causal maps \eqref{eq:nonlinear_of_policy} from measurement histories to controls, and hence is infinite-dimensional unless parameterized.
    \item The constraints \eqref{eq:nonlinear_of_constraint} must hold for each of the infinitely-many possible disturbances $w_k$, $e_k$ in the sets $\mathcal{B}^{\nx}$, $\mathcal{B}^{\ny}$.
    \item The dynamics and observation functions $f$, $h$ are nonlinear and high-dimensional. In particular, the observations are assumed to be high-dimensional images, i.e., $y_k \in \mathbb{R}^{\ny}$ with $\ny$ in the order of $10^5$ or more.
\end{itemize}

In this paper, we develop a tractable pipeline, VISION-SLS, to address these challenges. We will first summarize background preliminaries on \ac{SLS} in Sec. \ref{sec:preliminaries} and overview our method in Sec. \ref{sec:method_overview}, which learns reduced observations and robustly synthesizes safe nonlinear output-feedback control (Sec. \ref{sec:method}) via an efficient solver based on Riccati recursions and sequential convex programming (Sec. \ref{sec:method_scp}).

\section{Preliminaries}\label{sec:preliminaries}

In this section, we review the \ac{SLS} framework for linear time-varying systems \cite{DBLP:journals/arc/AndersonDLM19},  which we leverage as a subroutine in Section~\ref{sec:method_sls} to design a dynamic output-feedback controller for the nonlinear system \eqref{eq:system}. 
We consider a one-step measurement-delayed \ac{LTV} system to capture perception latency typical of vision-based feedback
\begin{equation}
    \begin{aligned}
\label{eq:ltv}
x_0 &= \bar x_0 + \Xi \tilde w,\\
x_{k+1} &= A_k x_k + B_k u_k + E_k w_k,\\
 y_{k+1} &= C_k x_k + F_k e_k,
    \end{aligned}
\end{equation}
where $\tilde w \in \mathcal{B}^{\nx}$ denotes a normalized initial state uncertainty, and $\Xi \in \mathbb{R}^{\nx \times \nx}$ denotes the corresponding scaling matrix.
The corresponding nominal dynamics evolve according to
\begin{equation}
\label{eq:ltv_nominal}
\bar{x}_{k+1} = A_k \bar{x}_k + B_k \bar{u}_k, \quad \bar{y}_{k+1} = C_k \bar{x}_k.
\end{equation}

\noindent We define the stacked variables
$\mathbf{x}=(x_0,\ldots,x_T)$, 
$\mathbf{u}=(u_0,\ldots,u_{T-1})$, 
$\mathbf{y}=(y_1,\ldots,y_{T})$, 
$\mathbf{w}=(\tilde w,w_0,\ldots,w_{T-1})$, 
and $\mathbf{e}=(e_0,\ldots,e_{T-1})$,
and block-diagonal matrices as
$\A=\mathrm{blkdiag}(A_0,\ldots,A_{T-1})$, and similarly for $\B$, $\C$, and $\F$,
and $\E = \mathrm{blkdiag}(\Xi, E_0,\ldots,E_{T-1})$. We further define deviations
$\Delta\mathbf{x} := \mathbf{x}-\bar{\mathbf{x}}$,
$\Delta\mathbf{u} := \mathbf{u}-\bar{\mathbf{u}}$,
$\Delta\mathbf{y} := \mathbf{y}-\bar{\mathbf{y}}$, and by stacking the dynamics over the time horizon, the closed-loop error system is written compactly as
\begin{subequations}
\label{eq:stacked_LTV}
\begin{align}
\Delta\mathbf{x} &= \Z \A \Delta\mathbf{x} + \Z \B \,\Delta\mathbf{u} + \E \mathbf{w},\\
\Delta\mathbf{y} &= \Z \C \Delta\mathbf{x} + \F \mathbf{e}.
\end{align}
\end{subequations}
We consider an output-feedback controller of the form
\begin{equation}
\label{eq:output_feedback}
\mathbf{u} =\bar {\mathbf{u}}  + \K (\mathbf{y} - \Z \C \bar {\mathbf{x}}).
\end{equation}
Optimizing $\K$ directly is difficult since enforcing robust constraint satisfaction for the realized trajectory is nonconvex in $\K$. In contrast, \ac{SLS} designs the closed-loop response maps $\bm \Phi$ directly, mapping disturbances to state and input deviations
\begin{equation}\label{eq:closed_loop_phi}
	\begin{bmatrix} \Delta \mathbf{x} \\ \Delta \mathbf{u} \end{bmatrix} = \underbrace{\begin{bmatrix} \Phixw & \Phixe \\ \Phiuw & \Phiue \end{bmatrix}}_{\mathbf{\Phi}} 
    \begin{bmatrix}
        \mathbf{E} &\mathbf{0} \\\mathbf{0}& \mathbf{F}
    \end{bmatrix}
    \begin{bmatrix} \mathbf{w} \\ \mathbf{e} \end{bmatrix},
\end{equation}
where the matrices $\Phixw$, $\Phixe$, $\Phiuw$, and $\Phiue$ are block-lower-triangular.
Crucially, this output-feedback parameterization enables an efficient \textit{convex} controller synthesis with robust constraint satisfaction, as we will see next.

\begin{proposition}[Adapted from \cite{DBLP:journals/arc/AndersonDLM19, zhou2023safe}]
\label{prop:slp}
The error LTV system \eqref{eq:stacked_LTV} in closed-loop with the controller \eqref{eq:output_feedback} can be written as Eq. \eqref{eq:closed_loop_phi}
\footnote{In the stacked matrices $\Phixw$ and $\Phiuw$, the block at index $j=0$ corresponds to the initial condition, while indices $j \ge 1$ correspond to process disturbances injected at time $j-1$.} if and only if $\mathbf{\Phi}$ satisfies

\begin{equation}\label{eq:slp_convex}
    \begin{aligned}
\begin{bmatrix} \mathbf{I} - \Z\A & -\Z \B \end{bmatrix} \mathbf{\Phi} &= \begin{bmatrix}\mathbf{I} & \mathbf{0}\end{bmatrix},\\
 \mathbf{\Phi} \begin{bmatrix} \mathbf{I} - \Z \A \\ -\Z\C \end{bmatrix} &= \begin{bmatrix} \mathbf{I} \\ \mathbf{0} \end{bmatrix}.        
    \end{aligned}
\end{equation}
The feedback gains $\K$ in \eqref{eq:output_feedback} can be recovered from $\mathbf{\Phi}$ via
\begin{equation}
\K = \Phiue - \Phiuw ({\Phixw})^{-1} \Phixe.
\end{equation}
\end{proposition}
\begin{proof}
See, e.g., \cite[Prop. 1]{zhou2023safe}.
\end{proof}

We show that the associated \ac{LQG} problem \cite{DBLP:journals/arc/AndersonDLM19,wang2016localized}
can be written as the convex program
\begin{subequations}\label{eq:lqg}
	\begin{align}
		\underset{\bm\Phi}{\textrm{min}} &\quad  
    \left\|
    \begin{bmatrix}
        \mathbf{Q}^{1/2} & \mathbf{0}\\
         \mathbf{0} &\mathbf{R}^{1/2} \\
    \end{bmatrix}
    \begin{bmatrix}
        \bm{\Phi}^{\mathrm{xw}} & \bm{\Phi}^{\mathrm{xe}}\\
        \bm{\Phi}^{\mathrm{uw}} & \bm{\Phi}^{\mathrm{ue}}\\
    \end{bmatrix}
    \begin{bmatrix}
        \mathbf{E}\\
        \mathbf{F}
    \end{bmatrix}
    \right\|_\Frob^2\nonumber\\
    & \qquad + \left\|  {P}^{1/2}   \begin{bmatrix}
        \bm{\Phi}^{\mathrm{xw}}_T & \bm{\Phi}^{\mathrm{xe}}_T\\
    \end{bmatrix}    \begin{bmatrix}
        \mathbf{E}\\
        \mathbf{F}
    \end{bmatrix}
    \right\|_\Frob^2\\
  		\textrm{s.t.} &\quad
\begin{bmatrix} \mathbf{I} - \Z\A & -\Z \B \end{bmatrix}     \begin{bmatrix}
        \bm{\Phi}^{\mathrm{xw}} & \bm{\Phi}^{\mathrm{xe}}\\
        \bm{\Phi}^{\mathrm{uw}} & \bm{\Phi}^{\mathrm{ue}}\\
    \end{bmatrix} = \begin{bmatrix}\mathbf{I} & \mathbf{0}\end{bmatrix},\\
&\quad     \begin{bmatrix}
        \bm{\Phi}^{\mathrm{xw}} & \bm{\Phi}^{\mathrm{xe}}\\
        \bm{\Phi}^{\mathrm{uw}} & \bm{\Phi}^{\mathrm{ue}}\\
    \end{bmatrix}\begin{bmatrix} \mathbf{I} - \Z \A \\ -\Z\C \end{bmatrix} = \begin{bmatrix} \mathbf{I} \\ \mathbf{0} \end{bmatrix},
	\end{align}
\end{subequations}
with $\mathbf{Q} = \mathrm{blkdiag}(Q, \ldots, Q)$ and $\mathbf{R} = \mathrm{blkdiag}(R, \ldots, R)$ the state and input cost matrices, respectively, with $Q \succeq 0$ and $R \succ 0$, and $P$ the terminal cost corresponding to the last time step $T$, and last block-row $\bm{\Phi}^{\bullet}_T$ of $\bm{\Phi}^{\bullet}$. As for typical LQG design, this cost function in \eqref{eq:lqg} corresponds to the expected value of the variance of the quadratic performance objective under stochastic process and measurement noise \cite{DBLP:journals/arc/AndersonDLM19}. In the following, we formulate the corresponding robust constraint satisfaction conditions for the output-feedback setting.

\begin{proposition}
\label{prop:linear_robust_constraint_satisfaction}
\looseness-1Consider \eqref{eq:stacked_LTV} in closed loop with \eqref{eq:output_feedback}, and let $\Phixw,\Phixe,\Phiuw,\Phiue$ satisfy \eqref{eq:slp_convex}. Then robust satisfaction of \eqref{eq:constraints} for all $k$ and all admissible disturbances holds if and only if
\begin{equation}
    \label{eq:robust_constraint}
    \begin{aligned}
            \hspace{-4pt}\sum_{j=0}^{k} \| c_i^{\top} \bm{\Phi}^{\w}_{k,j} E_j \|_1\hspace{-2pt} +\hspace{-2pt} \| c_i^{\top} \bm{\Phi}^{\e}_{k,j} F_j \|_1 \hspace{-2pt}+\hspace{-2pt} c_i^{\top} (z_k, v_k) + b_i \le 0.
    \end{aligned}
\end{equation}
for $i = 1, \ldots, n_\textrm{c}$ with $\bm{\Phi}^\w_{k,j} = (\bm{\Phi}^{\x\w}_{k,j}, \bm{\Phi}^{\u\w}_{k,j})$ and $\bm{\Phi}^\e_{k,j} = (\bm{\Phi}^{\x\e}_{k,j}, \bm{\Phi}^{\u\e}_{k,j})$.
\end{proposition}
\begin{proof}
    The proof follows directly from the state-feedback result in \cite{sieber2021system}; we provide the full proof in Appendix~\ref{appA_proof_SLC}.
\end{proof}

\section{Method Overview}\label{sec:method_overview}

In this section, we summarize our method for approximately solving the nonlinear output-feedback problem \eqref{eq:nonlinear_of} (also see Fig.~\ref{fig:overview}), which we call VISION-SLS. To make solving \eqref{eq:nonlinear_of} tractable, we introduce two approximations. First, instead of planning directly from high-dimensional images, we map each observation $y \in \mathbb{R}^{\ny}$ to a low-dimensional reduced measurement $y^r \in \mathbb{R}^{n_r}$. Concretely, we compose a pre-trained DINO feature extractor \cite{DBLP:journals/corr/abs-2304-07193} with a learned projection to obtain
\[
\proj^\text{dino}(\cdot): \mathbb{R}^{\ny} \rightarrow \mathbb{R}^{n_r}
\]
together with a corresponding state-based reduced observation model
\[
h^r(\cdot): \mathbb{R}^{\nx} \rightarrow \mathbb{R}^{n_r},
\]
so that the reduced observation can be written as
\begin{equation}
    y_{k+1}^r = \proj^\text{dino}(y_k)= h^r(x_k) + F^r(x_k)e_k.
    \label{eq:reduced_obs_map}
\end{equation}
Intuitively, \(h^r(x_k)\) captures the control-relevant information in the image, while \(F^r(x_k)e_k\) is an inflated, state-dependent uncertainty term that overbounds both sensor noise and the error introduced by compressing the image into a low-dimensional observation. This lets us plan efficiently in a reduced feature space, rather than directly in pixel space (see Sec.~\ref{sec:method_error_bound}).
\begin{remark}
\label{rem:lidar}
Although tailored for camera sensors, this work could be implemented with other sensors, such as LiDAR (see Sec. \ref{sec:results_baselines}), depth sensors, or joint encoders.
\end{remark}

Second, since optimizing over general causal policies is intractable, we restrict \eqref{eq:nonlinear_of_policy} to causal \emph{affine}, time-varying feedback policies around a jointly-optimized nominal state/control/measurement trajectory $(\bm{z}, \bm{v})$ \footnote{We use the notation $(\mathbf{\bar x}, \mathbf{\bar u})$ to denote the nominal trajectory of the LTV dynamics, in contrast to $(\bm{z}, \bm{v})$ which denote the nominal trajectory of the nonlinear dynamics.}, i.e.,
\begin{equation}\label{eq:affine} u_k = \bar{u}_k +K^0_k (x_0 - \bar x_0 ) + \sum_{j = 1}^{k-1} K_{k, j} (y_{j+1}^r - h^r(\bar x_j)), \end{equation}
with optimized gains $K_{k,j}$ (Sec.~\ref{sec:method_sls}). 
The additional term $K^0_k (x_0 - \bar x_0)$ explicitly accounts for uncertainty in the initial condition $\bar x_0$, which cannot in general be inferred from the measurement history and therefore enters the closed-loop system as an independent disturbance channel.
This form admits fast evaluation at runtime via matrix–vector products.

\begin{figure}[ht!]
    \centering
    \includegraphics[width=\linewidth]{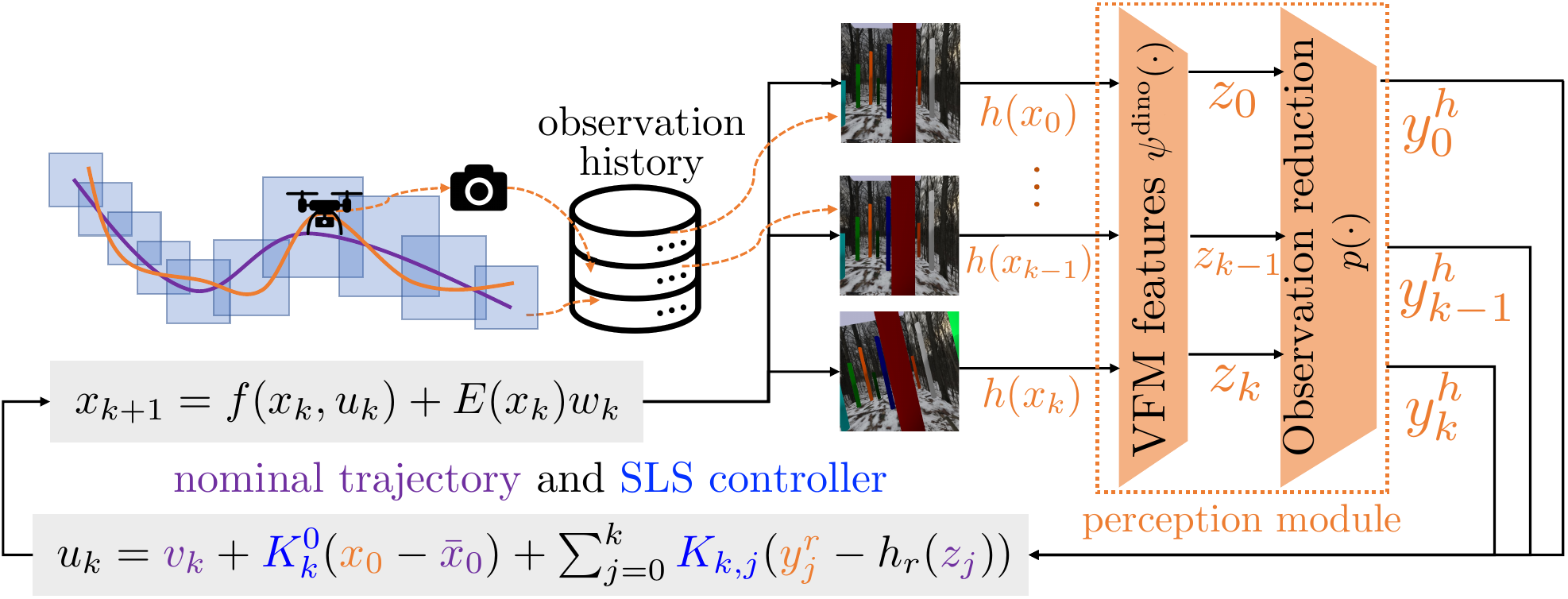}\vspace{-5pt}
    \caption{\textbf{Online SLS execution:} Images are mapped to low-dimensional observations 
$p(\psi^\textrm{dino}(y))$, approximated by a learned linear map $h^r(x) = Cx$ with bounded error. 
These observations feed into an \ac{SLS}-based dynamic output-feedback controller,
which combines the nonlinear nominal trajectory $(\bm{z}, \bm{v})$ and feedback gains $K$ to generate safe inputs.}
    \label{fig:overview}
\end{figure}

We detail our method in Sec. \ref{sec:method} and \ref{sec:method_scp}. In Sec.~\ref{sec:method_sls}, we develop a nonlinear \ac{SLS} approach to the output-feedback policy in \eqref{eq:nonlinear_of_policy} by combining nonlinear trajectory optimization with an \ac{SLS} parameterization of the local closed-loop error dynamics, with robust constraint satisfaction (Sec. \ref{sec:robust_constraint_satisfaction}).
In Sec.~\ref{sec:method_dino}, we use pre-trained visual features to learn a low-dimensional observation map $h^r$, together with error bounds in Sec.~\ref{sec:method_error_bound}, so that pixel observations can be used while retaining robust constraint satisfaction guarantees in Sec. ~\ref{sec:robust_constraint_satisfaction}. In Sec.~\ref{sec:method_scp}, we propose an efficient iterative convexification algorithm to solve the resulting nonconvex program efficiently.

\section{Nonlinear Control Design}\label{sec:method}

In this section, we build upon the linear LTV \ac{SLS} framework from Propositions~\ref{prop:slp}--\ref{prop:linear_robust_constraint_satisfaction} to synthesize a dynamic output-feedback controller for the nonlinear system~\eqref{eq:system} (Sec. \ref{sec:method_sls}) with robust constraint satisfaction guarantees (Sec. \ref{sec:robust_constraint_satisfaction}). 
Specifically, we jointly optimize the nominal trajectory $(\bm z,\bm v)$, the closed-loop responses $\mathbf{\Phi}$ of the induced LTV system, and modeling-error overbounds for robust constraint satisfaction of the nonlinear system. Finally, we tractably implement the proposed output-feedback SLS framework for visual measurements by learning a reduced observation model (Sec. \ref{sec:method_dino}).

\subsection{Closed-loop error dynamics}
\label{sec:method_sls}

Our robust prediction leverages a nominal system
\begin{equation}
    z_{k+1} = f(z_k, v_k),\quad k =0, \ldots, T-1
    \label{eq:nom_nonlinear}
\end{equation}
where $z_k\in\R^{\nx}$ is the nominal state, and $v_k\in\R^{\nuu}$ is the nominal input. 
We approximate the tracking error dynamics  
\begin{equation}
    \Delta x_{k+1} \doteq f(x_k, u_k) + E(x_k)w_k - f( z_k, v_k)
\end{equation}  
around the nominal trajectory $(\mathbf{ z}, \mathbf{v})$ using a first-order Taylor expansion.  
This yields the LTV error system  
\begin{equation}
    \label{eq:linearized_dynamics}
    \Delta x_{k+1} = A(z_k,v_k) \Delta x_k + B_k(z_k,v_k) \Delta u_k  
    + \sigma(\tau_k, z_k, v_k),
\end{equation}  
\begin{align}
    \hspace{-5pt}\textrm{with } A_k := \frac{\partial f(x,u)}{\partial x}\Big|_{(z_k, v_k)},\quad
    B_k := \frac{\partial f(x,u)}{\partial u}\Big|_{(z_k, v_k)},
\end{align}  
and the term $\sigma(\tau_k, z_k, v_k)$ captures the modelling error with $\tau\in \R$ satisfying
\begin{equation}
    \left\|
    \begin{bmatrix}
        \Delta x \\
        \Delta u
    \end{bmatrix}
    \right\|_\infty \le \tau,\quad \Delta x = x-z,\quad \Delta u = u-v.
    \label{eq:tau_definition}
\end{equation}

\begin{assumption}\label{ass:lin_error}
For each $k\in\mathbb{N}_{0:T-1}$ and any deviations $(\Delta x,\Delta u)$ satisfying
$\|(\Delta x,\Delta u)\|_\infty \le \tau_k$, the dynamics remainder is bounded as
\begin{equation}\label{eq:dyn_remainder_bound}
\begin{aligned}
|
f(z_k+\Delta x,\, v_k&+\Delta u)
- f(z_k,v_k) \\
&- A_k\Delta x - B_k\Delta u
|
\le \sigma(\tau_k,z_k,v_k).
\end{aligned}
\end{equation}
\end{assumption}

\begin{remark}
    The function $\sigma(\tau_k, z_k, v_k)$ can be constructed to capture the linearization error of the dynamics $f$ around the nominal trajectory, as detailed in \cite{Leeman_2025_TAC,leeman2025guaranteed} or directly learned from data~\cite{anutam2026}.
\end{remark}
Similarly, the reduced measurement model \eqref{eq:reduced_obs_map}
with measurement noise $e_k$ can be expanded around the nominal trajectory.
In the following, we use a linear reduced observation model
\begin{equation}
h^r(x)=C^r x,
\qquad C^r \in \mathbb{R}^{n_r \times \nx},    
\end{equation}
which yields the reduced measurement error model
\begin{equation}
\label{eq:measurement_error}
    \Delta y_{k+1}^r = C^r \Delta x_k + F^r(z_k)e_k + \eta(\tau_k).
\end{equation}
where \(\eta(\tau_k)\) captures the residual approximation error due to nonlinearities in $F^r(\cdot)$.

\begin{remark}
    \label{assumption:linear_obs}
        The reduced observation model \(h^r\) in \eqref{eq:reduced_obs_map} may be any differentiable nonlinear function. A key strength of the linear reduced observation model \eqref{eq:measurement_error} is that the learned perception map \(\proj(\cdot)\) absorbs the main nonlinearities of visual perception offline, so that online planning can rely on a simple linear reduced model, with the remaining mismatch captured by the calibrated state-dependent error bound in Sec.~\ref{sec:method_error_bound}.
\end{remark}

To certify the nonlinear output-feedback loop, we require that the residual reduced-observation mismatch admits a state- and tube-dependent overbound. This is stated in Assumption \ref{assum:bounded_perception_error} below and is
implemented in Sec.~\ref{sec:method_error_bound} via an empirical calibration process.

\begin{assumption}
\label{assum:bounded_perception_error}
For each $k \in \mathbb{N}_{0:T-1}$ and any deviations
$(\Delta x,\Delta u)$ satisfying $\|(\Delta x,\Delta u)\|_\infty \le \tau_k$,
the reduced observation mismatch in \eqref{eq:measurement_error} is bounded by the calibrated
state-dependent overbound $\Upsilon(\tau_k,z_k)$, i.e.,
\begin{equation}
    \big|F^r(z_k)e_k + \eta(\tau_k)\big| \le \Upsilon(z_k,\tau_k)e_k,
\end{equation}
where the absolute value is understood component-wise.
\end{assumption}

In Section~\ref{sec:method_error_bound}, we work with a reduced image $y^r$ and a corresponding reduced observation model \eqref{eq:measurement_error}. The function $F$ is defined to explicitly capture the approximation error introduced by this reduction, as described in \eqref{eq:reduced_obs_map}.

We then define the closed-loop response variables $\mathbf{\Phi}$ and enforce the SLS constraints \eqref{eq:slp_convex}. In this formulation, the system matrices $\A(\bm{z},\bm{v})$ and $\B(\bm{z},\bm{v})$ depend on the nominal decision variables $\bm{z}$ and $\bm{v}$. 
Specifically, the deviation variables satisfy
 \begin{equation}
     \begin{bmatrix}
         \mathbf{\Delta x}\\
         \mathbf{\Delta u}
     \end{bmatrix} = \mathbf{\Phi}
    \begin{bmatrix}
       {\mathbf{\Sigma}} & \mathbf{0}  \\ \mathbf{0}  & \mathbf{\Upsilon}
    \end{bmatrix} \begin{bmatrix}\mathbf{ w} \\ \mathbf{e} \end{bmatrix} 
 \end{equation}
which corresponds to the closed-loop response of the linearized dynamics \eqref{eq:linearized_dynamics}. Here,
${\mathbf{\Sigma}} = \text{blkdiag}({\Xi}, \Sigma_0, \ldots, \Sigma_{T-1})$ captures state-dependent disturbance scaling induced by the function $\sigma(\tau_k, z_k, v_k)$, with
\begin{equation}
    \begin{aligned}
        \Sigma_j &= \sigma(\tau_j, z_j, v_j) I_{\nx}, \quad j \in \mathbb{N}_{0:T-1},\\
    \end{aligned}
\end{equation}
including the initial condition uncertainty $\Xi$.
The vector $\bm{\tau} = (\rho, \tau_0, \ldots, \tau_{T-1})$ collects the tube radii, and $\mathbf{\Upsilon}$ is defined analogously to capture the observation uncertainty.

\subsection{Nonlinear robust constraint satisfaction}
\label{sec:robust_constraint_satisfaction}
Based on the LTV controller parameterization and the error overbounds, we can now enforce robust constraint satisfaction of the nonlinear system \eqref{eq:system}.

\begin{proposition}
\label{prop:robust_constraint_nonlinear}
Consider the output–feedback \ac{SLS} responses
\(
\Phixw,\Phixe,\Phiuw,\Phiue
\)
satisfying \eqref{eq:slp_convex} for the LTV error model and suppose that
Assumptions~\ref{ass:calibrated_uncertainty},~\ref{ass:lin_error} and \ref{assum:bounded_perception_error} hold. Then the closed-loop trajectories of the nonlinear system \eqref{eq:system} satisfy the state and input constraints robustly if the following tightened inequalities hold:
\begin{subequations}
    \begin{align}
    \label{eq:nonlin_constraint_tight}
    &\textstyle\sum_{j=0}^{k}
    \big\| c_i^{\top} \bm\Phi^\mathrm{w}_{k,j} \Sigma_j \big\|_1
+
\big\| c_i^{\top} \bm\Phi^\mathrm{e}_{k,j} \Upsilon_j \big\|_1
+
c_i^{\top} (z_k, v_k) + b_i \le 0.\\
    & \textstyle\sum_{j=0}^{k} \big\| \hat e_l^\top \bm\Phi^\mathrm{w}_{k,j} \Sigma_j \big\|_1 +\big\| \hat e_l^\top \bm\Phi^\mathrm{e}_{k,j} \Upsilon_j \big\|_1 \le \tau_k,
    \label{eq:nonlin_constraint_tight_tau}
    \end{align}
\end{subequations}
for all $\kinTm$, $i = 1, \ldots, n_\mathrm{c}$, with $\Sigma_j = \sigma(\tau_j, z_j, v_j) I_{\nx}$,
for $l=1,\ldots,\nx+  \nuu $ where \(\hat e_l\) denote standard basis vectors.

\end{proposition}
\begin{proof}
The proof follows directly from \cite{sieber2021system,Leeman_2025_TAC}. See Appendix \ref{appendix:proof_robust_constraint_nonlinear} for details.
\end{proof}

Collecting the nominal dynamics \eqref{eq:nom_nonlinear}, the controller parameterization \eqref{eq:slp_convex} based on the linearized dynamics \eqref{eq:linearized_dynamics} and observation model, and the robust constraint satisfaction in \eqref{eq:nonlin_constraint_tight}--\eqref{eq:nonlin_constraint_tight_tau}, we pose the following finite-horizon nonlinear output-feedback \ac{SLS} synthesis problem in Eq. \eqref{eq:nonlinear_of_sls}.
\begin{subequations}
\label{eq:nonlinear_of_sls}
\begin{align}
\underset{\substack{\mathbf{\Phi},\bm \tau  \\\bm z,\bm v}}{\text{min}}
&\; J(\bm{z}, \bm{v}, \bm{\tau}, \mathbf{\Phi}) \label{eq:nonlinear_of_sls_cost}\\
\textrm{s.t.}\quad & z_{k+1} = f(z_k, v_k), \quad \kinTm, 
\label{eq:nonlinear_of_sls_dyn}\\
& \bigl[\Iblk - \Z \A(\bm{z},\bm{v}) \;\; -\Z \B(\bm{z},\bm{v})\bigr]\;\mathbf{\Phi}
= [\Iblk\;\;\mathbf{0} ], 
\label{eq:nonlinear_of_sls_correctness1}\\
& \mathbf{\Phi}\,\begin{bmatrix} \Iblk - \Z \A(\bm{z},\bm{v}) \\[4pt] -\Z \C(\bm{z},\bm{v}) \end{bmatrix}
= \begin{bmatrix}\Iblk \\\mathbf{0} \end{bmatrix}, 
\label{eq:nonlinear_of_sls_correctness2}\\
& \sum_{j=0}^{k}\Bigl\|c_i^\top \bm\Phi^\w_{k,j}\,\Sigma_j\Bigr\|_1
+ \Bigl\|c_i^\top \bm\Phi^\e_{k,j}\,\Upsilon_j\Bigr\|_1
+ c_i^\top\begin{bmatrix} z_k \\ v_k \end{bmatrix} + b_i \le 0, \notag
\\
& \quad i=1,\ldots,n_\mathrm{c},\;\;\kinTm, \label{eq:nonlinear_of_sls_tube} \\
& \sum_{j=0}^{k}\Bigl\|\hat e_l^\top \bm\Phi^\w_{k,j}\,\Sigma_j\Bigr\|_1
+ \Bigl\|\hat e_l^\top \bm\Phi^\e_{k,j}\,\Upsilon_j\Bigr\|_1 \le \tau_k, 
\label{eq:nonlinear_of_sls_tau}\\
& \quad l=1,\ldots,\nx+\nuu,\;\;\kinTm. \notag
\end{align}
\end{subequations}
The formulation~\eqref{eq:nonlinear_of_sls} integrates all components of the robust prediction and control design:
\begin{itemize}
    \item The nominal trajectory is propagated according to the nonlinear dynamics model~\eqref{eq:nonlinear_of_sls_dyn}.
\item 
The propagation of model and measurement disturbance is captured via \eqref{eq:nonlinear_of_sls_correctness1}–\eqref{eq:nonlinear_of_sls_correctness2}, ensuring that the response matrices $\mathbf{\Phi}$ correspond to the closed-loop response of the LTV system linearized around the nominal trajectory.
\item Robust constraint satisfaction is enforced by the tightened constraints~\eqref{eq:nonlinear_of_sls_tube}, which guarantee that constraints \eqref{eq:constraints} hold for all disturbances and modeling errors.
\end{itemize}
As per Prop.~\ref{prop:robust_constraint_nonlinear}, the disturbance-feedback gains resulting from solving \eqref{eq:nonlinear_of_sls} ensures that the uncertain nonlinear system \eqref{eq:nonlinear_of_sls_dyn} remains within the constraints $\mathcal{S}$ over the prediction horizon.

\subsection{Perception function}
\label{sec:method_dino}
We now describe how to build the reduced observation map in \eqref{eq:reduced_obs_map} from camera output using pre-trained visual features. The goal is to compress high-dimensional images into a low-dimensional, control-relevant observation, while fitting a calibrated overbound on the error introduced by this reduction, so that robust constraint satisfaction can still be enforced via Prop.~\ref{prop:robust_constraint_nonlinear}.
\subsubsection{Learning a lower-dimensional visual representation}\label{sec:method_dino}
We seek a reduced representation that retains only the information in the image needed for control.  
Since directly planning from high-dimensional visual features is intractable, we instead project them into a low-dimensional observation space that is as informative as possible about the underlying state.
To quantify this, we use a dynamics-aware observability metric. Based on the dynamics \(f\) and the reduced observation model \(h^r\), we define the local discrete-time nonlinear observability matrix at \(\hat x\) over an input sequence \(\bm u := \{u_0,\ldots,u_n\}\) as
\begin{equation}\label{eq:observability}
    \mathcal{O}(\hat x, \bm u) \doteq \begin{bmatrix}
        h^r(f(\hat x, u_0)) \\
        h^r(f(f(\hat x, u_0), u_1)) \\
        \vdots \\
        h^r(f(f(\ldots), u_{n-1}), u_n))
    \end{bmatrix},
\end{equation}
The observability matrix $\mathcal{O}$ quantifies how well $x$ can be inferred from outputs: higher (Jacobian) rank implies better observability\cite[Ch.~3]{Grossman99}, implemented here via the minimum singular value $\sigma_\text{min}$.

\begin{remark}
If $\nabla_x \mathcal{O}\vert_{x=\hat x}$ is full rank, the system is strongly observable, i.e., $x$ is locally recoverable from $n\!+\!1$ outputs.
\end{remark}

\begin{remark}
Dimensionality reduction via reconstruction (e.g., autoencoders~\cite{DBLP:journals/jmlr/Baldi12}) often retains visually salient but control-irrelevant distractors~\cite{DBLP:conf/iclr/0001MCGL21} making overbounds loose.
\end{remark}

Given DINO features $\psi^\textrm{dino}(y_i)$, we learn two objects: a projection $\proj(\cdot)$ from the feature space to a low-dimensional observation space $\Y_r$, and a state-based reduced observation model $h^r(\cdot)$. The projection is trained so that the reduced observation $\proj(\psi^\textrm{dino}(y_i))$ matches the state-based target $h^r(x_i)$. To make this reduced observation useful for control, we also encourage it to preserve observability of the underlying state through the metric above. Specifically, over the dataset $\D$, we train $\proj(\cdot)$ and $h^r$ by minimizing
\begin{equation}\label{eq:loss_observability}
    \ell(\theta_p, \theta_h) = \Vert \proj(\psi^\textrm{dino}(y_i); \theta_p) - h^r(x_i; \theta_h) \Vert - \lambda \sigma_\text{min}(\mathcal{O}(\hat x, \bm u))
\end{equation}
where \(\theta_p\) are the parameters of the projection network and \(\theta_h\) parameterizes the reduced observation model \(h^r\) (in our implementation, the entries of \(C^r\)). The scalar \(\lambda > 0\) balances feature matching against observability. The dimension of the reduced observation \(h^r\) and the horizon length \(n\) used in \(\mathcal{O}\) are treated as hyperparameters, and \(\hat x\) and \(\bm u\) are randomly sampled during training. After training, we fix \(\proj(\cdot)\) and \(h^r(\cdot)\), and then calibrate a state-dependent overbound on the residual reduction error in Sec.~\ref{sec:method_error_bound}.

\subsubsection{Calibration of the perception reduction error}
\label{sec:method_error_bound}
Finally, once the perception module is fixed, we bound the residual error introduced by replacing the original high-dimensional observation with the reduced observation \(y^r=\proj(\psi^\textrm{dino}(y))\). 

This residual captures the mismatch between the learned reduced observation and its state-based model \(h^r(x)\). To empirically preserve robust constraint satisfaction in \eqref{eq:reduced_obs_map}, we overbound this mismatch by a state-dependent uncertainty set. Specifically, from a dataset $\{(x_i,y_i)\}$ we form residuals
\begin{equation}
    r_i \;=\; \big\| \proj(\psi^\textrm{dino}(y_i)) - h^r(x_i) \big\|_\infty .
\end{equation}
To define a state-dependent uncertainty set of the form \(F^r(x)\mathcal B\) that bounds this mismatch, we empirically calibrate a scalar envelope \(b(x)\) by fitting a polynomial via the convex program
\begin{equation}
  \begin{aligned}
  \min_{\beta,\,\xi_i\ge 0} \quad 
  & \sum_{i} \big( \beta^\top m(x_i) \big)
  + \gamma \|\beta\|_2^2 + \mu \sum_i \xi_i \\
  \text{s.t.}\quad
  & \beta^\top m(x_i) \;\ge\; r_i - \xi_i \quad \forall i ,
  \end{aligned}
\label{eq:optimal_polynomial}
\end{equation}
where $\gamma,\mu\!\ge\!0$ trade off smoothness and outlier violations, \(m(\cdot)\) is a monomial basis, and \(b(x)\!:=\!\beta^\top m(x)\) is the calibrated scalar envelope.
We then set $F^r(x) \;=\; b(x)\, I$ so that
\begin{equation}
    \proj(\psi^\textrm{dino}(y)) - h^r(x) \;\in\; F^r(x)\,\mathcal B_\infty.
\end{equation}
This procedure yields a differentiable, state-dependent uncertainty set that is cheap to evaluate inside constrained trajectory optimizers.

\begin{remark}
The calibrated scalar overbound is intended to certify perception error over the operating region covered by the calibration data, and is therefore most reliable near that distribution. Adapting the certificate to a new environment only requires recalibrating the scalar bound $b(x)$, rather than retraining the full perception-control pipeline. In our experiments (Section~\ref{sec:results}), this recalibration is data-efficient: even a small number of calibration samples already provides high empirical coverage. The learned perception map plays a different role: because it is built on DINO features that capture transferable geometric structure, it may generalize beyond the calibration environment, see, e.g, \cite{zhou2024dino}, while the bound is used only for safety certification.
Alternatively, Lipschitz-type bounds~\cite{knuth2021planning, chou2021model} and bounds given by neural network verifiers ~\cite{zhang2018efficient} can provide rigorous overbounding guarantees, but typically lead to nonsmooth or overly conservative constraint tightenings. Alternatively, distribution-free approaches such as conformal prediction~\cite{anutam2026} or extreme value theory~\cite{knuth2023statistical} can provide \textit{statistical} guarantees beyond the sampled data.     
\end{remark}

\section{Implementation via SCP}
\label{sec:method_scp}

In this section, we will discuss how we efficiently solve \eqref{eq:nonlinear_of_sls} for high-dimensional robotics problems. We first discuss an efficient Riccati-based solver for the LTV SLS output feedback problem \eqref{eq:lqg} (Sec. \ref{sec:riccati_solver}), which will be used as a subroutine in a sequential convex programming solver (Sec. \ref{sec:method_scp_alg}), which we call VISION-SLS.

\subsection{Efficient Solver for SLS-based Output-Feedback LQG}\label{sec:riccati_solver}
Although the LTV SLS output-feedback problem \eqref{eq:lqg} is convex and therefore amenable to off-the-shelf general-purpose convex optimization solvers \cite{gurobi}, the large number of decision variables and constraints causes these solvers to scale poorly with the horizon and system dimension in SLS problems~\cite{LEEMAN2024_fastSLS}. This can become prohibitive for large-scale robotics problems. 
To overcome this challenge, we demonstrate that the unconstrained linear time-varying problem~\eqref{eq:lqg} admits an efficient double-Riccati structure, akin to state-feedback~\cite{LEEMAN2024_fastSLS}. In Section~\ref{sec:method_scp}, this
structure is used as a computational building block inside the nonlinear SCP algorithm.
\begin{proposition}
    \label{prop:riccati}
Problem~\eqref{eq:lqg} has an optimal solution given by $T$ parallel backward Riccati recursions
\begin{equation}
    \begin{aligned}
    \label{eq:riccati}
    S_{T,j} & = P, \\
    K_{k,j} & =-\left(R+B_k^{\top} S_{k+1,j}B_k\right)^{-1}\left(B_k^{\top} S_{k+1,j}A_k\right), \\
    S_{k,j} & = Q+A_k^{\top} S_{k+1,j}A_k+\left(A_k^{\top} S_{k+1,j}B_k\right) K_{k,j},
    \end{aligned}
\end{equation}
and independently $T$ parallel backward Kalman recursions
\begin{equation}
    \begin{aligned}
    \label{eq:kalman_riccati_indices}
    \Pi_{k,0} & = \Xi \Xi^\top, \\
    L_{k,j+1}  &= - \Big(F_{j}F_{j}^\top + C_{j} \Pi_{k,j} C_{j}^\top \Big)^{-1} C_{j} \Pi_{k,j} A_{j}^\top, \\
    \Pi_{k,j+1} &= E_{j}E_{j}^\top + A_{j} \Pi_{k,j} A_{j}^\top + A_{j} \Pi_{k,j} C_{j}^\top L_{k,j+1},
    \end{aligned}
\end{equation}
where $S$ and $\Pi$ are the cost-to-go and estimation error covariance matrices, respectively, defined recursively.
Given the controller gains ${K_{k,j}}$ and observer gains ${L_{k,j}}$, compute the closed-loop propagation operators in parallel via the two $T$ forward passes
\begin{equation}
\label{eq:lyap_forward}
\begin{aligned}
\bar{\P}^\x_{j,j} &= I,\\
\bar{\P}^\u_{k,j} &= K_{k,j}\bar{\P}^\x_{k,j},\quad
\bar{\P}^\x_{k+1,j} = (A_k + B_k K_{k,j})\bar{\P}^\x_{k,j},
\end{aligned}
\end{equation}
    and
\begin{equation} \label{eq:riccati_obs}
    \begin{aligned}
       & \hat{\P}^\x_{k,k} = I,\\
       & \hat{\P}^\y_{k,j} = \hat{\P}^\x_{k,j} L_{k,j}^\top,\quad \hat{\P}^\x_{k,j} = \hat{\P}^\x_{k,j+1}A_j + \hat{\P}^\y_{k,j+1}C_j,
    \end{aligned}
\end{equation}
for $j = 0,\ldots, T,~ k=j, \ldots, T-1$ and $k=0,\ldots, T,~ j=0,\ldots,k-1$ respectively.

Finally, assemble the matrices using $\mathbf{M} = \mathbf{I}-\mathbf{Z}\mathbf{A}$, 
\begin{subequations}\label{eq:slp_fast}
\begin{align}
\mathbf{\Phi}^{\mathrm{xw}} &= \bar{\boldsymbol{\Phi}}^{\mathrm{x}}
 + \hat{\boldsymbol{\Phi}}^{\mathrm{x}}
 - \bar{\boldsymbol{\Phi}}^{\mathrm{x}}\mathbf{M}\hat{\boldsymbol{\Phi}}^{\mathrm{x}},\quad
\label{eq:slp_fast_1}
\mathbf{\Phi}^{\mathrm{uw}} = \bar{\boldsymbol{\Phi}}^{\mathrm{u}}(I
 - \mathbf{M}\hat{\boldsymbol{\Phi}}^{\mathrm{x}}),\\
\mathbf{\Phi}^{\mathrm{xe}} &= \hat{\boldsymbol{\Phi}}^{\mathrm{y}}
 - \bar{\boldsymbol{\Phi}}^{\mathrm{x}}\mathbf{M}\hat{\boldsymbol{\Phi}}^{\mathrm{y}}, \quad
\mathbf{\Phi}^{\mathrm{ue}} = -\,\bar{\boldsymbol{\Phi}}^{\mathrm{u}}
\mathbf{M}\hat{\boldsymbol{\Phi}}^{\mathrm{y}}.
\label{eq:slp_fast_4}
\end{align}
\end{subequations}

\end{proposition}
\begin{proof}
See Appendix \ref{appA_proof_SLP}.
\end{proof}
\noindent In the nonlinear constrained setting, we use this double-Riccati structure as an efficient computational primitive inside SCP; we discuss the resulting nonlinear algorithm in Sec. ~\ref{sec:method_scp_alg}.

\subsection{Implementation via SCP}\label{sec:method_scp_alg}

To solve~\eqref{eq:nonlinear_of_sls}, we adopt an \ac{SCP} approach~\cite{verschueren2022acados} that, at each iteration, constructs a convex \ac{SOCP} by linearizing the dynamics and convexifying the cost and constraint tightening around the current nominal trajectory $(\bm z,\bm v)$. Each subproblem is then solved using a structure-exploiting \ac{SLS} method based on~\cite{LEEMAN2024_fastSLS}, together with the double-Riccati construction of Proposition~\ref{prop:riccati} to construct feasible output-feedback responses $\Phi$ and their reachable sets. 
In practice, this amounts to combining the control-side state-feedback SLS solution with a time-varying Kalman recursion for state estimation. Because the extra estimation step only requires a standard Riccati recursion, the per-iteration complexity remains close to that of a single state-feedback fast-\ac{SLS} \ac{MPC} step. We note that the overall runtime of this SCP procedure can be reduced by leveraging GPU parallelization, following \cite{fang2026safe}. The resulting scheme is efficient and maintains feasibility, though we note the coupling is ultimately heuristic and does not yield optimality guarantees. In particular, as opposed to the state-feedback fast-SLS solver \cite{LEEMAN2024_fastSLS}, the dual-based cost update used within the unconstrained Riccati recursions \eqref{eq:lqg} is not based on the KKT structure of the constrained problem (analogous to \cite[Eq. 21]{LEEMAN2024_fastSLS}), and is hence a heuristic. Future work should focus on uncovering the banded numerical structure of the constrained SLS-based LQG. The full procedure is summarized in Algorithm~1.

\begin{tcolorbox}[title=\textit{Algorithm 1}: SCP-based Output-Feedback SLS,boxsep=0.1mm,left=2mm,right=3mm]
\textbf{Repeat until convergence (SCP loop):}
\begin{enumerate}
\item \textbf{Linearize} nonlinear dynamics/constraints around $(\bm z,\bm v)$; quadratize cost to obtain a convex \ac{SOCP} subproblem.\label{alg:initialize}
\item \textbf{Solve linear approximation:}
\begin{enumerate}
\item solve QP: update $(\Delta \bm z,\Delta \bm v)$,\label{alg:QP}
\item compute $\bar{\bm{\Phi}}^\x,\bar{\bm{\Phi}}^\u$ via fast-SLS \cite{LEEMAN2024_fastSLS} (parallel Riccati recursions).
\item compute $\hat{\bm{\Phi}}^\x,\hat{\bm{\Phi}}^\y$ (parallel Kalman recursions).
\item compute $\Phixw$, $\Phiuw$, $\Phixe$, $\Phiue$ via \eqref{eq:slp_fast} 
\item compute constraint tightening \eqref{eq:nonlin_constraint_tight}--\eqref{eq:nonlin_constraint_tight_tau}
\item set $(\bm z,\bm v)\leftarrow (\bm z,\bm v)+(\Delta \bm z,\Delta \bm v)$
\end{enumerate}
\item \textbf{Repeat} 
\end{enumerate}
\end{tcolorbox}

\section{Results}
\label{sec:results}

\looseness-1We evaluate our output-feedback \ac{SLS} pipeline on four systems with uncertain nonlinear dynamics and uncertain partial observations: a 4D car with overhead RGB images (Sec. \ref{sec:results_car}), a 10D quadrotor with onboard RGB images (Sec. \ref{sec:results_quad}), a humanoid with partial state sensing (Sec. \ref{sec:results_humanoid}), and a real-world TurtleBot with onboard RGB images (Sec. \ref{sec:results_hardware}). For intuition, we also evaluate on the classic light-dark benchmark \cite{DBLP:conf/rss/PlattTKL10} (Sec. \ref{sec:results_lightdark}) and perform a baseline comparison against a recent safe perception-based control algorithm \cite{yang2023safeperceptionbasedcontrolstochastic} (Sec. \ref{sec:results_baselines}). We demonstrate (i) robust constraint satisfaction via tube tightening, (ii) information-gathering motion emerging from the \ac{SLS} design even for large state dimensions and RGB observations. For the car and quadrotor, we use DINOv2-vitg14 and DINOv2-vitl14 \cite{DBLP:journals/corr/abs-2304-07193} as the pre-trained VFM, leveraging 512x512 RGB images rendered in PyBullet, respectively with $3\times 10^5$ and $2\times 10^6$ datapoints to train $\proj(\cdot)$ and $C^r$, and fit a polynomial error bound with \eqref{eq:optimal_polynomial}. 
To implement Alg. 1, we use CasADi \cite{andersson2019casadi}, using IPOPT \cite{wachter2006implementation} to find an initial guess for $(\mathbf{z}, \mathbf{v})$ (Step \ref{alg:initialize} of Alg. \ref{alg:initialize}) and OSQP \cite{stellato2020osqp} to solve the QP in Step \ref{alg:QP} of Alg. 1 \footnote{The implementation and experiment code are available at \href{https://github.com/trustworthyrobotics/VISION-SLS}{https://github.com/trustworthyrobotics/VISION-SLS}.}. Alg. 1 terminates if $\Vert (\Delta z, \Delta v)\Vert_2 \leq 0.001$ or after 20 iterations.
We compare Alg. 1 to several baselines, including 1) a certainty equivalent (CE) controller \cite{dean2021certainty, todorov2005generalized}, which assumes the separation principle, akin to iLQG, by optimizing a reference and controller under the assumption of perfect observations, i.e., with post-hoc tubes, which corresponds to the first iteration of our algorithm
2) a non-robust (NR) belief-space controller inspired by \cite{DBLP:conf/rss/PlattTKL10}, which performs information gathering but does not perform constraint tightening to robustly ensure safety under uncertainty.
We define success as reaching the terminal set without collision and report success rate (SR) and constraint violation rate (CVR) of our method and baselines over 15 random initial states each with 50 rollouts under random disturbance sequence in Table \ref{table:1}.   

{\small
\begin{table}[h!]
\centering
    \begin{tabular}{|c|c|c|c|c|c|c|}
        \hline
        ~ & \multicolumn{3}{|c|}{Light-dark} & \multicolumn{3}{|c|}{4D Car} \\ \hline
        ~ & Ours & CE & NR & Ours & CE & NR \\ \hline
        SR & \textbf{100\%} & 90.13\% & 87.73\% & \textbf{98.53\%} & 93.47\% & 48.13\%  \\
        CVR & \textbf{0\%} & 9.87\% & 12.27\% & \textbf{1.6\%} & 6.53\% & 55.87\%  \\ \hline
        ~ & \multicolumn{3}{|c|}{Quadrotor} & \multicolumn{3}{|c|}{Humanoid} \\\hline
        ~ & Ours & CE & NR & Ours & CE & NR \\\hline
        SR & \textbf{100\%} & 100\% & 29.47\% & \textbf{100\%} & 100\% & 2\% \\
        CVR & \textbf{0\%} & 0\% & 92.8\% & \textbf{0\%} & 0\% & 98\% \\ \hline
    \end{tabular}
\caption{Comparison between our method, certainty equivalent (CE) and non-robust (NR) baseline in terms of success rate (SR) and constraint violation rate (CVR).\vspace{-20pt}}
\label{table:1}
\end{table}
}

\subsection{Safe Perception-Based Control Baseline}\label{sec:results_baselines}
To compare the performance of VISION-SLS compared to recent safe perception-based control methods, we compare with \cite{yang2023safeperceptionbasedcontrolstochastic} in their F1/10th autonomous driving setting \cite[Sec. 5]{yang2023safeperceptionbasedcontrolstochastic}, which assumes a 3D vehicle state $[p_x, p_y, \theta] \in \mathbb{R}^3$ consisting of translation $(p_x, p_y)$ and orientation $\theta$. We use the dataset generated by the code implementation of \cite{yang2023safeperceptionbasedcontrolstochastic} to train our perception module and error bound. Notably, \cite{yang2023safeperceptionbasedcontrolstochastic} leverages a 64-dimensional LiDAR observation instead of pixel observations; our method can flexibly be applied to this setting by learning a projection function $p: \mathbb{R}^{64} \rightarrow \mathbb{R}^3$, setting $h^r(x) = x$, and calibrating the error bound $F^r(x_k)$ using the procedure of \eqref{eq:optimal_polynomial}. VISION-SLS achieves 100\% safety, higher than the 93\% reported by \cite{yang2023safeperceptionbasedcontrolstochastic}, since our method jointly synthesizes a trajectory and robust controller that provides hard guarantees on safety given a valid overbound $F^r$.

\subsection{Light-Dark}
\label{sec:results_lightdark}

We first show the information-gathering abilities, runtime scaling with horizon length, and optimality of our method in the light-dark domain~\cite{DBLP:conf/rss/PlattTKL10}, a benchmark for partially-observed control. We consider a 2D single integrator with state $x = (p_x, p_y)$, with state, input, and terminal state constraints (see App. \ref{app:light_dark}).
Alg. 1 converges after 17 SCP iterations in $1.171$ seconds. 
 By penalizing the tube volume in \eqref{eq:nonlinear_of_sls_cost}, our method steers the system to lower perception-error regions (Fig. \ref{fig:light_dark}c), which reduces tube growth and enables information-gathering while satisfying the terminal constraints. In contrast, the CE baseline experiences larger perception error, causing larger tubes and terminal constraint violation. All $10$ simulated closed-loop rollouts stay in their computed reachable sets, as expected. Moreover, we show that our method scales much better than a na\"ive SLS output feedback controller synthesized via convex optimization  \cite{DBLP:journals/arc/AndersonDLM19, zhou2023safe} that does not use the double Riccati structure (Fig. \ref{fig:light_dark}a) while achieving the same cost (Fig. \ref{fig:light_dark}b), empirically validating the optimality result of Prop. \ref{prop:riccati}. Our method always reaches the terminal set while CE fails in 9.87\% and NR fails in 12.27\% of all trials (see Table \ref{table:1}).

\begin{figure}
    \centering
    \includegraphics[width=\linewidth]{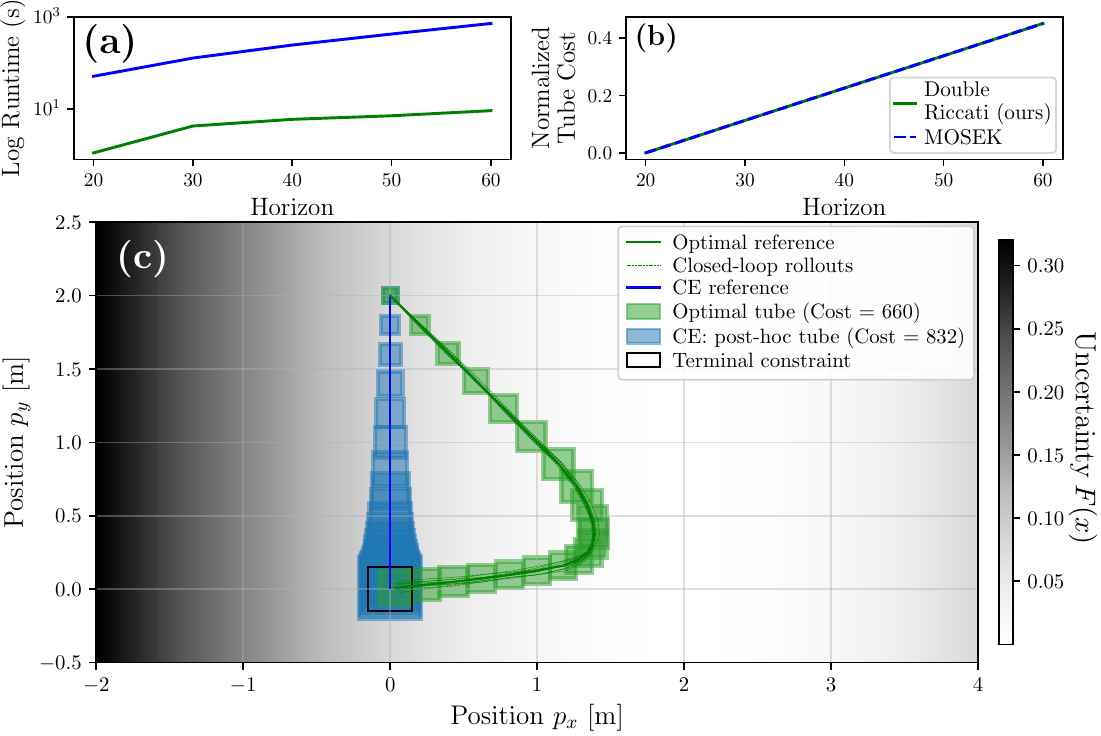}
    \vspace{-23pt}
    \caption{On the classic light-dark domain, our method optimally visits the low perception-error region, where lighter shades indicate lower error (c), reducing final state uncertainty, whereas a na\"ive CE controller accumulates more uncertainty over the horizon. Computationally, our method is much faster than a baseline convex optimization-based SLS solver \cite{DBLP:journals/arc/AndersonDLM19, zhou2023safe}, implemented via MOSEK \cite{mosek} (a) while achieving the same optimal solution cost (b).
    }\vspace{-10pt}
    \label{fig:light_dark}
\end{figure}

\subsection{4D Car with Overhead RGB Images}
\label{sec:results_car}

\begin{figure*}[!htbp]
    \centering
    \includegraphics[width=\linewidth]{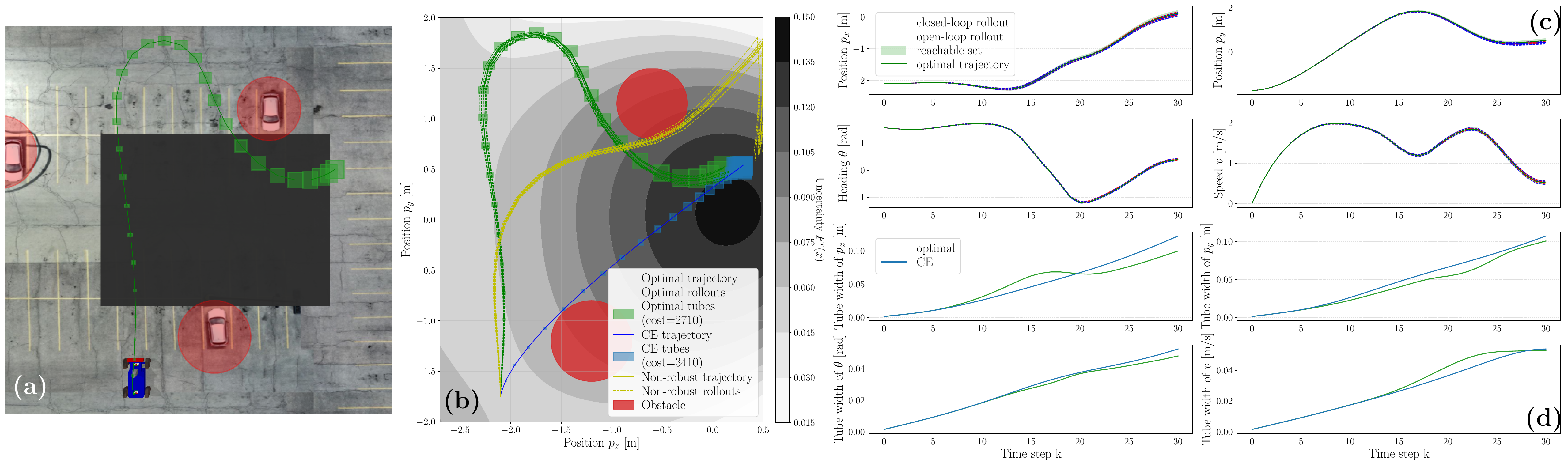}
    \vspace{-23pt}
    \caption{\looseness-1(a) Car in an occluded parking lot with a top-down camera avoiding red obstacles. (b) Grayscale shows the learned uncertainty for the reduced observation (darker = higher uncertainty/occlusion). Our controller tightens constraints using the propagated set and yields a collision-free trajectory. In contrast, the certainty-equivalent trajectory stays longer in the high-uncertainty occluded region. (c) Closed-loop rollouts in the occluded scene, overlaying the nominal plan, the \ac{SLS}-induced uncertainty set (for constraint tightenings), and sample closed-loop rollouts. (d) Comparing reachable tube size per state dimension. We minimize a weighted tube-volume cost, with a stronger penalty on the terminal tube, encouraging information gathering to reduce final-time uncertainty. 
    }\vspace{-10pt}
    \label{fig:unicycle_perception}
\end{figure*}

\looseness-1In the vision-based driving task shown in Fig.~\ref{fig:unicycle_perception}, we steer a 4D car with state $(p_x, p_y, p_\theta, v)$ (position, rotation, linear velocity) \cite{chou2021model} to a specified goal over a horizon $T=30$ with a timestep of $0.15$s, enforcing state and input constraints, avoiding collisions with other cars modeled as circular obstacles, while minimizing the propagated uncertainty along the trajectory. More details are in App. \ref{app:parking_lot}. 
To clearly show information-gathering in the vision-based control setting, a heavy visual occlusion (0.97 opacity) is placed in the middle of the environment, increasing the worst-case measurement error for those states. 
In Fig. \ref{fig:unicycle_perception}, we visualize the learned {state-dependent polynomial overbound} $F^r(x)$ as a grayscale heatmap over the $(p_x,p_y)$ workspace.  The calibration of the perception error bounds are given in App. \ref{perception_calibration_appendix}. 
The coverage of error bound over-approximation (tested at 500 randomly sampled states) vs. number of calibration points, presented in Table \ref{table:3}, demonstrates the data-efficiency of our perception calibration procedure. Notably, our method exceeds 90\% empirical coverage with only 50 samples.
The optimal plan produced by Alg.~1 (computed in 3.584s over 
$20$ SCP iterations) is shown in green in Fig.~\ref{fig:unicycle_perception}. The resulting trajectory avoids the occluded region to actively reduce perception uncertainty, whereas the CE baseline drives through the occlusion and incurs larger tube widths (Fig.~\ref{fig:unicycle_perception}d).
Closed-loop rollouts using DINO-v2 features remain in the computed reachable tubes and satisfy all constraints despite heavy occlusion (Fig.~\ref{fig:unicycle_perception}c). In contrast, the non-robust baseline violates safety under disturbance, highlighting the necessity of constraint tightening for vision-based control.

In this setting, we also compare our method against a variant of \cite{tong2023enforcingsafetyvisionbasedcontrollers} in which we replace the NeRF model with a convolutional neural network trained on our perception dataset to predict the minimum distance to obstacles (to better match our constrained single-camera-view setting). The SR and CVR of the variant of \cite{tong2023enforcingsafetyvisionbasedcontrollers} are 65.47\% and 27.93\% respectively, indicating our method outperforms all baselines with 98.53\% SR and 1.6\% CVR. The variant of \cite{tong2023enforcingsafetyvisionbasedcontrollers} performs poorly here because its formulation makes it difficult to sample safe actions for nonlinear dynamics. Moreover, \cite{tong2023enforcingsafetyvisionbasedcontrollers} does not formally account for robust constraint satisfaction under disturbances, and thus cannot counteract the error of the distance predictor. 

{\small
\begin{table}[h!]
\vspace{-3pt}
\centering
    \begin{tabular}{|c|c|c|c|c|c|c|}
    \hline
        calibration points & 50 & 100 & 250 & 500 & 1500 \\
        empirical coverage & 91.4\% & 92.2\% & 97.4\% & 99\% & 99\% \\ 
    \hline
    \end{tabular}
\caption{Error bound over-approximation coverage vs. number of calibration points.\vspace{-20pt}}
\label{table:3}
\end{table}
}

\subsection{10D Quadrotor with Onboard RGB Images}\label{sec:results_quad}

\begin{figure}
    \centering
    \includegraphics[width=\linewidth]{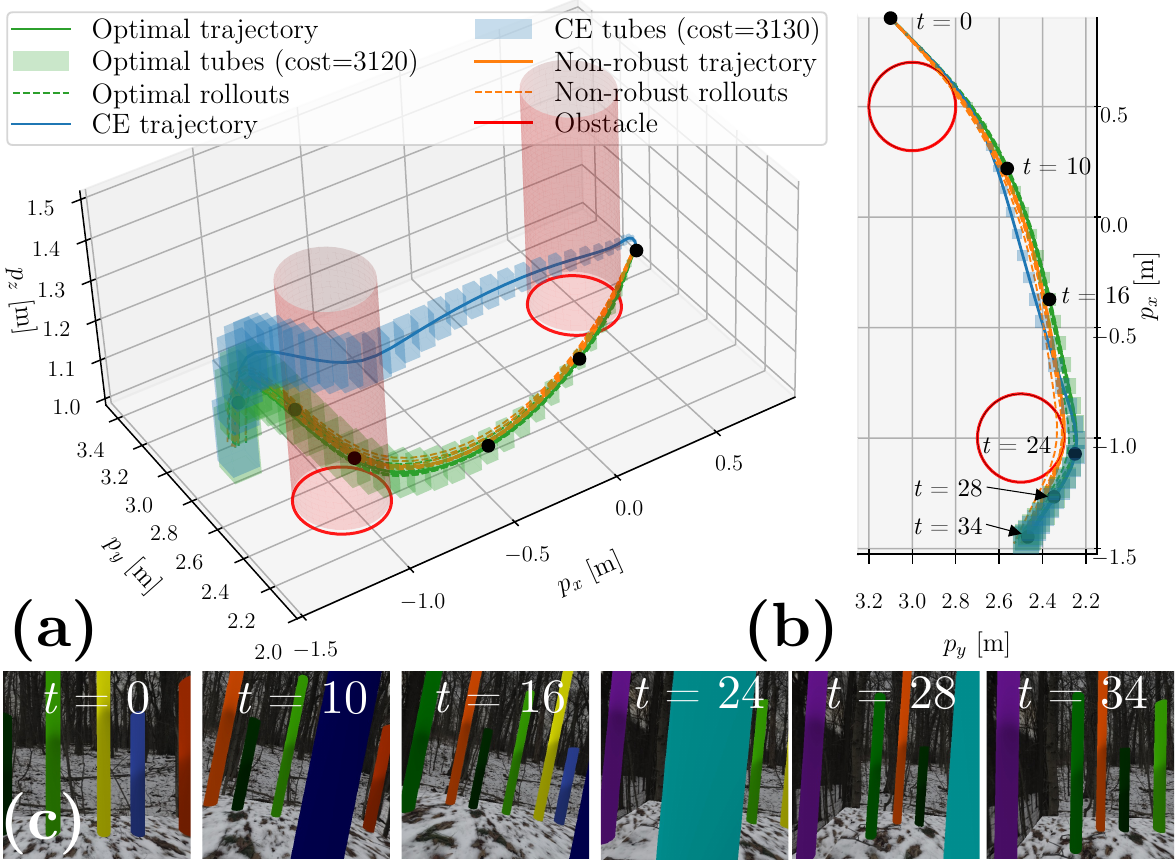}\vspace{-10pt}
    \caption{Controlling a quadcopter with an onboard camera, showcasing obstacle avoidance despite disturbances from perception error and dynamics uncertainty. (a) Our closed-loop rollouts stay within the computed tubes, reaching the target with smaller uncertainty bounds compared to the CE baseline, maintaining safety while executing a non-robust information-gathering policy does not. (b) A top-down view of (a). (c) Snapshots of the images collected onboard to implement feedback control.} \vspace{-5pt}
    \label{fig:quad}
\end{figure}

We evaluate on a vision-based flying task (Fig.~\ref{fig:quad}), steering a 10D quadrotor model~\cite{chou2021model} in a cluttered 3D environment to a goal under partial observability (see App.~\ref{app:quad}). The resulting prediction tubes robustly avoid obstacles despite perception uncertainty, whereas both the certainty-equivalent controller and the non-robust information-gathering baseline collide with obstacles, even though their nominal trajectories are collision-free. In addition, tight state constraints (e.g., pitch) are respected throughout the horizon due to robust constraint tightening. Our method succeeds and never violates constraints in all trials. The CE baseline also succeeds in all trials, but produces larger tubes because it cannot perform information gathering. In contrast, the NR baseline succeeds only in 29.47\% trials (see Table \ref{table:1}). 
Overall, these results show robust constraint satisfaction for vision-based control tasks on higher-dimensional dynamics. 

\subsection{Humanoid with Partial State Measurement}\label{sec:results_humanoid}

We demonstrate applicability to high-dimensional systems with partial state measurements on a Unitree G1 humanoid~\cite{kuindersma2016optimization} with 59 states and 23 inputs modeled in Pinocchio~\cite{carpentier2019pinocchio}. 
We consider a vision-free setting to isolate scalability with respect to the state dimension from scalability with respect to the observation dimension, since the former primarily impacts online control synthesis whereas the latter is handled offline by the perception reduction module (see Remark \ref{rem:lidar}).
Hence, only joint angles and base pose are observed (no vision), and state and input constraints must be satisfied (see App.~\ref{app:humanoid} for details and all reachable tubes). The task is robust tracking of a backflip trajectory under disturbance (Fig.~\ref{fig:humanoid_and_hardware}a). In Fig. \ref{fig:humanoid_and_hardware}b, open-loop rollouts (red) exit the reachable tubes and violate constraints (e.g., quaternion limits), whereas all closed-loop rollouts (blue) remain within the tubes and satisfy constraints, demonstrating effective output-feedback optimization and joint reachability–trajectory optimization for partially observed large-scale dynamics.
Our method always steers the humanoid to the terminal set without constraint violation despite the uncertainties and partial observation. In this example, the CE baseline also yields the same solution because the perception uncertainty is uniformly bounded over the space. However, the NR baseline performs far worse than VISION-SLS, achieving only 2\% SR and 98\% CVR (Table~\ref{table:1}), because even small disturbances can be amplified by the complex nonlinear humanoid dynamics.

\vspace{-2pt}
\subsection{Hardware experiment}\label{sec:results_hardware}
\vspace{-1pt}

We evaluate our method on a physical TurtleBot~4 platform operating in an office environment (Fig.~\ref{fig:humanoid_and_hardware}d). The robot is tasked with reaching a goal position under state and input constraints, including collision avoidance with furniture. 
See App. \ref{app:hardware} for detailed experiment description.
Using onboard visual feedback, the closed-loop controller robustly tracks the planned trajectory across $3$ real-world runs (Fig. \ref{fig:humanoid_and_hardware}c). 
Fig.~\ref{fig:humanoid_and_hardware}e shows representative onboard images used during execution.
All executions remain within the computed reachable tubes and satisfy collision-avoidance constraints, demonstrating safe perception-based output-feedback control on hardware.

\section{Conclusion}
We introduced a tractable framework for safe nonlinear output-feedback planning and control directly from pixels. The key ingredients are (i) a learned low-dimensional observation map that turns high-dimensional images into reduced measurements with calibrated error bounds, and (ii) an \ac{SLS}-based parameterization of dynamic output-feedback as causal affine policies around a nominal trajectory with robust constraint satisfaction.
After training the perception map and error bounds from data, the full nonlinear output-feedback synthesis problem is solved efficiently via SCP and an \ac{SLS} solver based on Riccati recursions. Empirically, the method enables information-gathering behavior and robust constraint satisfaction on high-dimensional, underactuated systems with realistic visual observations.
Overall, the results suggest that pairing visual representation learning with system-level synthesis is a practical path to safe visuomotor control at scale. Future work includes total occlusion handling and learned dynamics, and an optimal solver for the constrained version of \eqref{eq:lqg}.

\newpage
\bibliographystyle{plainnat}
\bibliography{IEEEexample}

\newpage
\appendix

\subsection{Proof of Proposition \ref{prop:linear_robust_constraint_satisfaction}}
\label{appA_proof_SLC}
\begin{proof}
    As per \eqref{eq:closed_loop_phi}, the state and input trajectories are affine images of $\ell_\infty$ balls, hence can be written as Minkowski sums
    \begin{equation}
    \begin{aligned}
        &\mathbf{x} \in \bar{\mathbf{x}} \oplus \Phixw \E \mathcal{B}^{\nx T} \oplus \Phixe \F\mathcal{B}^{\nx T},\\
         &\mathbf{u} \in \bar{\mathbf{u}} \oplus \Phiuw \E\mathcal{B}^{\nx T} \oplus \Phiue \F\mathcal{B}^{\nx T},
    \end{aligned}
\end{equation}
By \eqref{eq:closed_loop_phi}, $(\mathbf{x},\mathbf{u})$ are affine images of $\ell_\infty$ balls, i.e., Minkowski sums. Robust feasibility of \eqref{eq:constraints} is the support function of the disturbance set in direction $c_i$. Since the support function of a (scaled) $\ell_\infty$ ball equals a weighted $\ell_1$ norm, \eqref{eq:robust_constraint} follows.
\end{proof}

\subsection{Proof of Proposition \ref{prop:riccati}}
\label{appA_proof_SLP}
\begin{proof}
The proof proceeds in two parts. We first show that the closed-loop responses constructed in \eqref{eq:slp_fast} are feasible, i.e., they satisfy the \ac{SLS} constraints \eqref{eq:slp_convex}. We then show that, among all feasible responses, this construction is optimal for the quadratic objective \eqref{eq:lqg}.

\paragraph{Feasibility}
Define $\M \defeq \Iblk - \Z\A$. We show that the matrices in \eqref{eq:slp_fast} satisfy the \ac{SLS} constraints \eqref{eq:slp_convex} and are block-lower-triangular.

First, let $\bar{\bm{\Phi}}^{\x},\bar{\bm{\Phi}}^{\u}$ denote the state-feedback closed-loop responses associated with the gains $K_{k,j}$, constructed by the forward recursion \eqref{eq:lyap_forward}. By construction, they satisfy the stacked closed-loop relation
\begin{equation}\label{eq:bar_phi_id}
\M\,\bar{\bm{\Phi}}^{\x} - \Z\B\,\bar{\bm{\Phi}}^{\u} = \Iblk.
\end{equation}

Second, let $\hat{\bm{\Phi}}^{\x},\hat{\bm{\Phi}}^{\y}$ denote the observer-side operators associated with the gains $L_{k,j}$, constructed by the forward recursion in Proposition~\ref{prop:riccati}. These operators satisfy the right-inverse identity
\begin{equation}\label{eq:hat_phi_id}
\hat{\bm{\Phi}}^{\x}\,\M - \hat{\bm{\Phi}}^{\y}\,\Z\C = \Iblk .
\end{equation}

We now verify \eqref{eq:slp_convex} for the assembled matrices \eqref{eq:slp_fast}. Using \eqref{eq:slp_fast_1}--\eqref{eq:slp_fast_4}, \eqref{eq:bar_phi_id}, and \eqref{eq:hat_phi_id}, we get

\smallskip\noindent{(i) Left constraint, $\w$-channel in (\ref{eq:lqg}b):}
\begin{align*}
&\M \Phixw - \Z\B\,\Phiuw\\
&= \M\!\left(\bar{\bm{\Phi}}^{\x} + \hat{\bm{\Phi}}^{\x} - \bar{\bm{\Phi}}^{\x}\M\hat{\bm{\Phi}}^{\x}\right)
 - \Z\B\!\left(\bar{\bm{\Phi}}^{\u} - \bar{\bm{\Phi}}^{\u}\M\hat{\bm{\Phi}}^{\x}\right) \\
&= \left(\M\bar{\bm{\Phi}}^{\x} - \Z\B\bar{\bm{\Phi}}^{\u}\right)
+ \M\hat{\bm{\Phi}}^{\x}
- \left(\M\bar{\bm{\Phi}}^{\x} - \Z\B\bar{\bm{\Phi}}^{\u}\right)\M\hat{\bm{\Phi}}^{\x} \\
&= \Iblk + \M\hat{\bm{\Phi}}^{\x} - \Iblk\,\M\hat{\bm{\Phi}}^{\x}
= \Iblk .
\end{align*}

\smallskip\noindent{(ii) Left constraint, $\e$-channel in (\ref{eq:lqg}b):}
\begin{align*}
&\M \Phixe - \Z\B\,\Phiue\\
&= \M\!\left(\hat{\bm{\Phi}}^{\y} - \bar{\bm{\Phi}}^{\x}\M\hat{\bm{\Phi}}^{\y}\right)
 - \Z\B\!\left(-\bar{\bm{\Phi}}^{\u}\M\hat{\bm{\Phi}}^{\y}\right) \\
&= \M\hat{\bm{\Phi}}^{\y}
- \left(\M\bar{\bm{\Phi}}^{\x} - \Z\B\bar{\bm{\Phi}}^{\u}\right)\M\hat{\bm{\Phi}}^{\y} \\
&= \M\hat{\bm{\Phi}}^{\y} - \Iblk\,\M\hat{\bm{\Phi}}^{\y}
= \mathbf{0}.
\end{align*}

\smallskip\noindent{(iii) Right constraint, $\w$-channel in (\ref{eq:lqg}c):}
\begin{align*}
&\Phixw\,\M - \Phixe\,\Z\C\\
&= \left(\bar{\bm{\Phi}}^{\x} + \hat{\bm{\Phi}}^{\x} - \bar{\bm{\Phi}}^{\x}\M\hat{\bm{\Phi}}^{\x}\right)\M
 - \left(\hat{\bm{\Phi}}^{\y} - \bar{\bm{\Phi}}^{\x}\M\hat{\bm{\Phi}}^{\y}\right)\Z\C \\
&= \bar{\bm{\Phi}}^{\x}\M + \left(\hat{\bm{\Phi}}^{\x}\M - \hat{\bm{\Phi}}^{\y}\Z\C\right)
- \bar{\bm{\Phi}}^{\x}\M\left(\hat{\bm{\Phi}}^{\x}\M - \hat{\bm{\Phi}}^{\y}\Z\C\right) \\
&= \bar{\bm{\Phi}}^{\x}\M + \Iblk - \bar{\bm{\Phi}}^{\x}\M\,\Iblk
= \Iblk .
\end{align*}

\smallskip\noindent{(iv) Right constraint, $\e$-channel in (\ref{eq:lqg}c):}
\begin{align*}
&\Phiuw\,\M - \Phiue\,\Z\C\\
&= \left(\bar{\bm{\Phi}}^{\u} - \bar{\bm{\Phi}}^{\u}\M\hat{\bm{\Phi}}^{\x}\right)\M
 - \left(-\bar{\bm{\Phi}}^{\u}\M\hat{\bm{\Phi}}^{\y}\right)\Z\C \\
&= \bar{\bm{\Phi}}^{\u}\M - \bar{\bm{\Phi}}^{\u}\M\left(\hat{\bm{\Phi}}^{\x}\M - \hat{\bm{\Phi}}^{\y}\Z\C\right) \\
&= \bar{\bm{\Phi}}^{\u}\M - \bar{\bm{\Phi}}^{\u}\M\,\Iblk
= \mathbf{0}.
\end{align*}

Finally, $\bar{\bm{\Phi}}^{\x},\bar{\bm{\Phi}}^{\u},\hat{\bm{\Phi}}^{\x},\hat{\bm{\Phi}}^{\y}$ are block-lower-triangular by construction, and the set of block-lower-triangular operators is closed under addition and multiplication; hence $\Phixw,\Phiuw,\Phixe,\Phiue$ are block-lower-triangular. Therefore, the matrices defined by \eqref{eq:slp_fast} satisfy \eqref{eq:slp_convex}, and provide a feasible solution to Problem~\eqref{eq:lqg}.

\paragraph{Optimality}
We now show that this feasible solution is optimal for the quadratic objective.
We adopt a stochastic disturbance model in this section to leverage the classical LQG separation structure underlying the quadratic objective in \eqref{eq:lqg} under SLS parameterization.

The error dynamics following \eqref{eq:ltv} and \eqref{eq:ltv_nominal} is 
\begin{align}
\Delta x_{k+1} &= A_k \Delta x_k + B_k \Delta u_k + E_k w_k, \label{eq:lqg_dyn}\\
\Delta y_{k+1} &= C_k \Delta x_k + F_k e_k, \label{eq:lqg_obs}
\end{align}
where $\{w_k\}$ and $\{e_k\}$ are zero-mean, mutually independent stochastic disturbances with identity covariance, independent across time. 

The objective in \eqref{eq:lqg} corresponds to the expected finite-horizon quadratic cost \eqref{eq:lqg_expected_cost} since we can substitute $\Delta x_k$ and $\Delta u_k$ with $\Phixw, \Phixe, \Phiuw, \Phiue$ defined in \eqref{eq:closed_loop_phi}
\begin{equation}
J_{\mathrm{error}}=\mathbb{E}\!\left[
\sum_{k=0}^{T-1}
\begin{bmatrix} \Delta x_k \\ \Delta u_k \end{bmatrix}^{\!\top}
\!\begin{bmatrix} Q & 0 \\ 0 & R \end{bmatrix}
\!\begin{bmatrix} \Delta x_k \\ \Delta u_k \end{bmatrix}
+ \Delta x_T^\top P \Delta x_T
\right].
\label{eq:lqg_expected_cost}
\end{equation}

We decompose the state error into its estimate and estimation error,
\[
\Delta x_k = \Delta \hat x_k + d_k,
\]
where $\Delta \hat x_k := \mathbb{E}[\Delta x_k \mid y_{1:k}]$ and $d_k := \Delta x_k - \Delta \hat x_k$.

Substituting this decomposition into the quadratic cost yields
\begin{align}
\Delta x_k^\top Q \Delta x_k
&= (\Delta\hat x_k + d_k)^\top Q (\Delta\hat x_k + d_k) \nonumber\\
&= \Delta\hat x_k^\top Q \Delta\hat x_k
+ d_k^\top Q d_k
+ 2 \Delta\hat x_k^\top Q d_k.
\label{eq:quad_split}
\end{align}
Taking expectation and using the orthogonality of $\Delta\hat x_k$ and $d_k$ under conditional expectation gives
\[
\mathbb{E}[\Delta\hat x_k^\top Q d_k] = 0,
\]
so the expected cost decomposes exactly as
\begin{equation}
J_{\mathrm{error}} = J_{\mathrm{ctrl}} + J_{\mathrm{obs}},
\label{eq:cost_decomposition}
\end{equation}
with
\begin{align}
J_{\mathrm{ctrl}}
&:= \mathbb{E}\!\left[
\sum_{k=0}^{T-1}
\begin{bmatrix} \Delta\hat x_k \\ \Delta u_k \end{bmatrix}^{\!\top}
\!\begin{bmatrix} Q & 0 \\ 0 & R \end{bmatrix}
\!\begin{bmatrix} \Delta\hat x_k \\ \Delta u_k \end{bmatrix}
+ \Delta\hat x_T^\top P \Delta\hat x_T
\right], \label{eq:J_ctrl}\\
J_{\mathrm{obs}}
&:= \mathbb{E}\!\left[
\sum_{k=0}^{T-1} d_k^\top Q d_k + d_T^\top P d_T
\right]. \label{eq:J_obs}
\end{align}

\textbf{Control subproblem.}
The state estimate evolves according to
\begin{align} \notag
\Delta\hat x_{k+1}
& = A_k \Delta\hat x_k + B_k \Delta u_k
+ L_{k+1}\!\left(\Delta y_{k+1} - C_k \Delta\hat x_k\right),\\
\Delta\hat x_{k+1}
& = A_k \Delta\hat x_k + B_k \Delta u_k
+ L_{k+1}(C_kd_k + F_ke_k),
\label{eq:estimator_dyn}\\
d_{k+1} & = \Delta x_{k+1} - \Delta \hat{x}_{k+1} \\ \notag
& = A_k d_k  - L_{k+1} C_{k} d_{k} + E_k w_k - L_{k+1} F_{k}e_{k}, 
\end{align}
which follows from \eqref{eq:lqg_dyn}–\eqref{eq:lqg_obs} and the definition of $\Delta\hat x_k$.
Conditioned on the observations, \eqref{eq:estimator_dyn} is linear in $(\Delta\hat x_k, \Delta u_k)$
and $d_k$ is a self-evolving variable that is not controlled by $u_k$.
Minimizing $J_{\mathrm{ctrl}}$ over causal policies $\Delta u_k = K_k \Delta\hat x_k$ is therefore a finite-horizon LQR problem.

Equivalently, if we treat $L_{k+1}(C_kd_k + F_ke_k)$ as some disturbance independent of $x_k$ and $u_k$, in the system-level parameterization, this problem can be written as a convex quadratic optimization over closed-loop responses $(\mathbf{\bar\Phi}^x, \mathbf{\bar\Phi}^u)$: 

\begin{align} \label{eq:controller_optimization}
    \min_{\mathbf{\bar\Phi}^x, \mathbf{\bar\Phi}^u} & 
    \left\|
    \begin{bmatrix}
        \mathbf{Q}^{1/2} & \mathbf{0}\\
        \mathbf{0} & \mathbf{R}^{1/2}
    \end{bmatrix}
    \begin{bmatrix}
        \mathbf{\bar\Phi}^{x}\\
        \mathbf{\bar\Phi}^{u}
    \end{bmatrix}
    \right\|_\Frob^2
    +\left\|P^{1/2}\mathbf{\bar\Phi}^{x}_T \right\|_\Frob^2 \\ \notag
    \text{s.t. } & 
    \begin{bmatrix}
        \mathbf{I}-\mathbf{ZA} & -\mathbf{ZB}
    \end{bmatrix}
    \begin{bmatrix}
        \mathbf{\bar\Phi}^{x}\\
        \mathbf{\bar\Phi}^{u}
    \end{bmatrix}
    = \mathbf{I}
\end{align}
The unique optimal solution is characterized by a backward Riccati recursion, as in \cite{LEEMAN2024_fastSLS}.

\textbf{Estimation subproblem.}
The estimation error evolves as
\begin{equation}
d_{k+1}
= (A_k - L_{k+1} C_k) d_k
+ E_k w_k - L_{k+1} F_k e_k.
\label{eq:error_dyn}
\end{equation}

With the stacked variables, we have 
\begin{align} \notag
    \mathbf{d} = \mathbf{ZAd} - \mathbf{LZCd} + \mathbf{w} - \mathbf{Le} = \mathbf{\hat\Phi}^x \mathbf{w} + \mathbf{\hat\Phi}^y \mathbf{e}. 
\end{align}
Then 
\begin{align} \notag
    \mathbf{\hat\Phi}^{x} = & \mathbf{(I - ZA + ZLZC)}^{-1}, \\ \notag 
    \mathbf{\hat\Phi}^{y} = & -\mathbf{(I - ZA + ZLZC)}^{-1}\mathbf{ZL}.
\end{align}

Substituting $d_k$ with $\mathbf{\hat \Phi}^x$ and $\mathbf{\hat \Phi}^y$ in \eqref{eq:J_obs}, we get 
\begin{align} \notag
    J_{\textrm{obs}} = & \mathbb{E}\!\left[ \left\|
    \mathbf{Q}^{1/2} (\mathbf{\hat \Phi}^x \mathbf{E w} + \mathbf{\hat \Phi^y} \mathbf{F e})
    \right\|_\Frob^2 \right] + \\ \notag
    & \mathbb{E}\!\left[
    \left\|P^{1/2}(\mathbf{\hat{\Phi}}^{x}_T \mathbf{E w} + \mathbf{\hat \Phi}^y_T\mathbf{Fe})\right\|_\Frob^2 
    \right]
\end{align}
Minimizing $J_{\mathrm{obs}}$ corresponds to selecting the observer gains $\{L_k\}$, or equivalently the observer system responses $(\hat\Phi^x,\hat\Phi^y)$, to minimize the expected quadratic energy of $d_k$ as shown in (\ref{eq:observer_optimization}). 
This is a dual LQR problem whose solution is given by the (time-reversed) Kalman Riccati recursion.
\begin{align} \label{eq:observer_optimization}
    \min_{\mathbf{\hat\Phi}^x, \mathbf{\hat\Phi}^y} & \left\|
    \mathbf{Q}^{1/2} 
    \begin{bmatrix}
        \mathbf{\hat{\Phi}}^{x} & \mathbf{\hat{\Phi}}^{y}
    \end{bmatrix}
    \begin{bmatrix}
        \mathbf{E}\\
        \mathbf{F}
    \end{bmatrix}
    \right\|_\Frob^2
    +\left\|P^{1/2}\begin{bmatrix}
        \mathbf{\hat{\Phi}}^{x}_T & \mathbf{\hat{\Phi}}^{y}_T
    \end{bmatrix}
    \begin{bmatrix}
        \mathbf{E}\\
        \mathbf{F}
    \end{bmatrix}\right\|_\Frob^2 \\ \notag
    \text{s.t. } & 
    \begin{bmatrix}
        \mathbf{\hat{\Phi}}^{x} & \mathbf{\hat{\Phi}}^{y}
    \end{bmatrix}
    \begin{bmatrix}
        \mathbf{I - ZA}\\
        -\mathbf{ZC}
    \end{bmatrix}
    = \mathbf{I}
\end{align}

\textbf{Optimality and separation.}
Since the expected cost decomposes exactly as in \eqref{eq:cost_decomposition}, and the controller and observer subproblems depend on disjoint decision variables, the joint minimization of \eqref{eq:lqg_expected_cost} is achieved by independently solving \eqref{eq:J_ctrl} and \eqref{eq:J_obs}. As \eqref{eq:lqg_expected_cost} is equivalent to \eqref{eq:lqg}, 
the resulting output-feedback controller is therefore optimal and admits the classical double-Riccati characterization. Importantly, this stochastic interpretation is used solely to justify the structure and optimality of the unconstrained performance objective. The resulting controller is later combined with robust constraint tightenings that treat disturbances in a worst-case, bounded sense, ensuring that safety guarantees do not rely on stochastic assumptions.

This completes the proof.
\end{proof}

\subsection{Proof of Proposition \ref{prop:robust_constraint_nonlinear}}
\label{appendix:proof_robust_constraint_nonlinear}
\begin{proof}
We follow the same argument as Proposition~\ref{prop:linear_robust_constraint_satisfaction}, with the only change being that the disturbance scalings are now given by the nonlinear-model overbounds $\Sigma_j$ and $\Upsilon_j$, and the nominal point is $(z_k,v_k)$. By Assumptions~\ref{ass:lin_error} and \ref{assum:bounded_perception_error}, whenever
$\|(\Delta x_k,\Delta u_k)\|_\infty \le \tau_k$,
the nonlinear deviation dynamics are contained in the LTV error model
with additive disturbances scaled by $\Sigma_k=\sigma(\tau_k,z_k,v_k)I_{\nx}$, which includes the linearization error.

Fix any time $k$. By the SLS parameterization \eqref{eq:closed_loop_phi} applied to the LTV error model, the deviations satisfy
\[
\begin{bmatrix}\Delta x_k\\ \Delta u_k\end{bmatrix}
=
\sum_{j=0}^{k}\Phi^{\w}_{k,j}\Sigma_j\,w_j
+\sum_{j=0}^{k}\Phi^{\e}_{k,j}\Upsilon_j\,e_j.
\]
For any constraint row $c_i^\top(x_k,u_k)+b_i\le 0$,
\begin{align*}
c_i^\top(x_k,u_k)+b_i
&= c_i^\top(z_k,v_k)+b_i \\
&\quad + \sum_{j=0}^{k} c_i^\top\Phi^{\w}_{k,j}\Sigma_j\,w_j
      + \sum_{j=0}^{k} c_i^\top\Phi^{\e}_{k,j}\Upsilon_j\,e_j .
\end{align*}

Taking the worst case over $\|w_j\|_\infty\le 1$ and $\|e_j\|_\infty\le 1$ uses the same support-function identity as in Proposition~\ref{prop:linear_robust_constraint_satisfaction}:
\(
\max_{\|d\|_\infty\le 1} a^\top d=\|a\|_1.
\)
Therefore,
\begin{align*}
\max_{w_j,e_j}\big(c_i^\top(x_k,u_k)\big)
&= c_i^\top(z_k,v_k) \\
&\quad + \sum_{j=0}^{k}\|c_i^\top\Phi^{\w}_{k,j}\Sigma_j\|_1
      + \sum_{j=0}^{k}\|c_i^\top\Phi^{\e}_{k,j}\Upsilon_j\|_1 .
\end{align*}
which shows that \eqref{eq:nonlin_constraint_tight} is sufficient (and exact for this affine disturbance model) to guarantee robust satisfaction.

Then, the condition $\|(\Delta x_k,\Delta u_k)\|_\infty\le \tau_k$ is equivalent to
$|\hat e_\ell^\top(\Delta x_k,\Delta u_k)|\le \tau_k$ for all standard basis vectors $\hat e_\ell$.
Applying the same support-function argument coordinate-wise yields
\[
\max_{w_j,e_j}\hat e_\ell^\top\!\begin{bmatrix}\Delta x_k\\ \Delta u_k\end{bmatrix}
\le
\sum_{j=0}^{k}\Big(
\|\hat e_\ell^\top\Phi^{\w}_{k,j}\Sigma_j\|_1
+
\|\hat e_\ell^\top\Phi^{\e}_{k,j}\Upsilon_j\|_1
\Big),
\]
so enforcing \eqref{eq:nonlin_constraint_tight_tau} guarantees $\|(\Delta x_k,\Delta u_k)\|_\infty\le \tau_k$.

Combining both parts for all $k$ and all constraints $i$ completes the proof.
\end{proof}

\subsection{Experiment Details} 
We omit linearization error for simplicity within our implementation, as empirically, the resulting tubes maintain 100\% robust constraint satisfaction, following \cite{zhan2025robustly}. Linearization error can be reintroduced by following the procedure laid out in \cite{Leeman_2025_TAC}. In the following sections, we list the dynamics models as well as important parameters used in the experiments. 

\subsubsection{Cost Function}\label{app:cost_function}
The cost function (\ref{eq:nonlinear_of_sls_cost}) is the sum of a trajectory cost (\ref{eq:full_trajectory_cost}) and a tube cost (\ref{eq:full_tube_cost}) where $x^g$ is the goal state. In the humanoid experiment, we have a reference trajectory instead of a single goal state.

\begin{align} \label{eq:full_nonlinear_of_sls_cost}
    J(\bm{z}, \bm{v}, \bm{\tau}, \mathbf{\Phi}) = & J_{\textrm{traj}} + J_{\textrm{tube}} \\ \notag
    J_{\textrm{traj}}(\bm{z}, \bm{v}) = & \sum_{k=0}^{T-1} ( (z_k-x^g)^\top \bar Q (z_k-x^g) + v_k^\top \bar R v_k) \\ \label{eq:full_trajectory_cost}
    & + (z_T-x^g)^\top \bar P (z_T-x^g) \\ \notag
    J_{\textrm{tube}}(\bm{z}, \bm{v}, \bm{\tau}, \bm{\Phi}) =  
    & \left\|
    \begin{bmatrix}
        \mathbf{Q}^{1/2} & \mathbf{0}\\
         \mathbf{0} &\mathbf{R}^{1/2} \\
    \end{bmatrix}
    \begin{bmatrix}
        \bm{\Phi}^{\mathrm{xw}} & \bm{\Phi}^{\mathrm{xe}}\\
        \bm{\Phi}^{\mathrm{uw}} & \bm{\Phi}^{\mathrm{ue}}\\
    \end{bmatrix}
    \begin{bmatrix}
        \mathbf{E}\\
        \mathbf{F}
    \end{bmatrix}
    \right\|_\Frob^2 \\ \label{eq:full_tube_cost}
    &  
    + \left\|  {P}^{1/2}   \begin{bmatrix}
        \bm{\Phi}^{\mathrm{xw}}_T & \bm{\Phi}^{\mathrm{xe}}_T\\
    \end{bmatrix}    \begin{bmatrix}
        \mathbf{E}\\
        \mathbf{F}
    \end{bmatrix}
    \right\|_\Frob^2
\end{align}
where we use the shorthand notation $\mathbf{E}=\mathbf{E}(\bm{z,v,\tau})$ and $\mathbf{F}=\mathbf{F}(\bm{z,v,\tau})$. 
\subsubsection{Light-Dark}\label{app:light_dark}
We consider 2D single integrator dynamics: 
\begin{equation}
    {\begin{bmatrix}
        \dot p_x \\ \dot p_y
    \end{bmatrix}} = \begin{bmatrix}
        v_x \\ v_y
    \end{bmatrix}
\end{equation}
with state $x = (p_x, p_y) \in \mathbb{R}^2$ and control $u = (v_x, v_y) \in \mathbb{R}^2$. The constraints include state constraints ($x_k \in [-2, 5]\times[-2, 5]$) and control constraints ($u_k \in [-1,1]\times[-1,1]$), as well as terminal state constraints $x_T \in [-0.15,0.15]\times[-0.15,0.15]$. The initial nominal state is $\bar x_0 = (0, 2)$ and goal state is $x^g=(0, 0)$. 
The penalty terms in cost function are $\bar Q =\bar R=I_2, \bar P=50I_2, Q=R=P=10^4I_2$.
We consider initial condition uncertainty and dynamics uncertainty matrix $\Xi=E=0.05I_2$ and observation uncertainty $F(x)=0.02(p_x-2)^2$. The planning horizon is $T = 20$ with a timestep of $0.2$s. 

\subsubsection{Parking-Lot}\label{app:parking_lot}

We consider Dubins' car dynamics:
\begin{equation}\label{eq:dynamics_dubins}
    {\begin{bmatrix}
        \dot p_x \\ \dot p_y \\ \dot \theta \\ \dot v
    \end{bmatrix}} = \begin{bmatrix}
        v\cos\theta \\ v\sin\theta \\ u_1 \\ u_2
    \end{bmatrix}
\end{equation}
with x-position $p_x$, y-position $p_y$, orientation $\theta$, and linear velocity $v$ as states. The state constraints are $[-3.5, 0.5]\times[-2, 1.9]\times[-2\pi,2\pi]\times[-1,2]$ and input constraints are $[-\pi,\pi]\times[-4,4]$. The terminal state constraints are $[0,0.45]\times[0.35,0.65]\times[-2\pi,2\pi]\times[-1,1]$. The initial nominal state is $\bar x_0=(-2.1, -1.75, 0.5\pi, 0)$ and the goal state is $x^g=(0.25, 0.5, 0, 0)$. We consider initial condition uncertainty described by $\Xi=0.01I_4$ and dynamics uncertainty $E = 0.01I_{4}$. The perception uncertainty bound is given by the shading in Fig. \ref{fig:unicycle_perception}b. 
The penalty terms in cost function are $\bar Q=\textrm{diag}([1,1,0,0]), \bar R=0.5I_2,\bar P=\textrm{diag([100,100,0,10])}, Q=10^5I_4, R=10^5I_2, P=3\times 10^6I_4$.
The planning horizon is $T=30$, and the continuous-time dynamics of \eqref{eq:dynamics_dubins} are time-discretized with a timestep of $0.15$s. 

The mapping $p(\cdot)$ is a three-layer MLP with GELU as activation functions. The hidden dimensions are 512. 
The observation matrix $C^r$ learned is 
\begin{align} \notag
    \begin{bmatrix}
        0.976 & 0.132 & -0.172 &  0 \\
        -0.216 & 0.668 & -0.712 & 0\\
        0.0209 & 0.733 & 0.680 &  0
    \end{bmatrix}.
\end{align}
The monomial basis $m(\cdot)$ is third order and the states related to $F^r$ are $p_x$ and $p_y$. 

\subsubsection{Quadrotor}\label{app:quad}

We consider 3D quadrotor dynamics with ten states \cite{singh2018robust}:
\begin{equation}\label{eq:dynamics_quad}
    {\begin{bmatrix}
        \dot p_x \\ \dot p_y \\ \dot p_z \\ \dot v_x \\ \dot v_y \\ \dot v_z \\ \dot \theta_x \\ \dot \theta_y \\ \dot \omega_x \\ \dot \omega_y
    \end{bmatrix}} = \begin{bmatrix}
        v_x \\ v_y \\ v_z \\ g\tan(\theta_x) - 3v_x \\ g\tan(\theta_y) - 3v_y \\  u_z - g  - v_z\\ -10\theta_x + \omega_x \\ -10\theta_y + \omega_y \\ -10\theta_x + 50u_x \\ -10 \theta_y + 50 u_y
    \end{bmatrix}
\end{equation}
The state constraints are $x_T \in [-10, 10]\times[-10, 10]\times[0,10]\times[-5,5]\times[-5,5]\times[-5,5]\times[-2,2]\times[-2,2]\times[-2,2]\times[-2,2]$ and the control constraints are $u_k \in [-5,5]\times[-5,5]\times[0,40]$. 
The initial nominal state and goal state are respectively 
\begin{align} \notag
    \bar x_0 = & (0.9, 3.1, 1, 0.001, 0.001, 0.001, 0.001,\\ \notag 
    & -0.001, -0.001, -0.001) \\ \notag
    x^g = & (-1.5, 2.5, 1.4, 0.001, 0.001, 0.001, -0.001, \\ \notag
    & -0.001, -0.001, -0.001) .
\end{align}
We consider an initial state uncertainty and dynamics error of $\Xi=E=0.01I_{10}$ and a learned state-dependent perception uncertainty map. 
The penalty terms in cost function are $\bar Q=\text{diag}([1, 1, 1, 1, 1, 1, 1, 0, 0, 0]), \bar R=\text{diag}([1,1,0,1]), \bar P=\text{diag}([50, 50, 50, 5, 5, 5, 5, 1, 1, 1]), Q=P=5\times 10^5I_{10}, R=5\times 10^5I_3$. 
The planning horizon is $T=35$, with the continuous-time dynamics of \eqref{eq:dynamics_quad} time-discretized with a timestep of $0.15$s. The runtime of Alg. 1 is $5.592$ seconds over 10 SCP iterations.

We also plot the reachable tubes for each state in Fig. \ref{fig:quadrotor_tubes}.

The mapping $p(\cdot)$ is a nine-layer MLP with GELU as activation functions. The hidden dimensions are 1024. 
The learned observation matrix $C^r$ is 
\begin{align}\small \notag
    \begin{bmatrix}
        0.34 & -0.78 & 0.15 & 0 & 0 & 0 & -0.27 & -0.42 & 0 & 0\\
        -0.47 & -0.62 & -0.22 & 0 & 0 & 0 & 0.24 & 0.53 & 0 & 0\\
        0.78 & -0.04 & -0.33 & 0 & 0 & 0 & 0.42 & 0.32 & 0 & 0\\
        0.07 & -0.05 & 0.89 & 0 & 0 & 0 & 0.40 & 0.20 & 0 & 0\\
        0.21 & 0.04 & 0.18 & 0 & 0 & 0 & -0.72 & 0.63 & 0 & 0
    \end{bmatrix}.
\end{align}
The monomial basis $m(\cdot)$ is fourth order and the states related to $F^r$ are $p_x$, $p_y$, $p_z$, $\theta_x$ and $\theta_y$. 

\begin{figure}
    \centering
    \includegraphics[width=\linewidth]{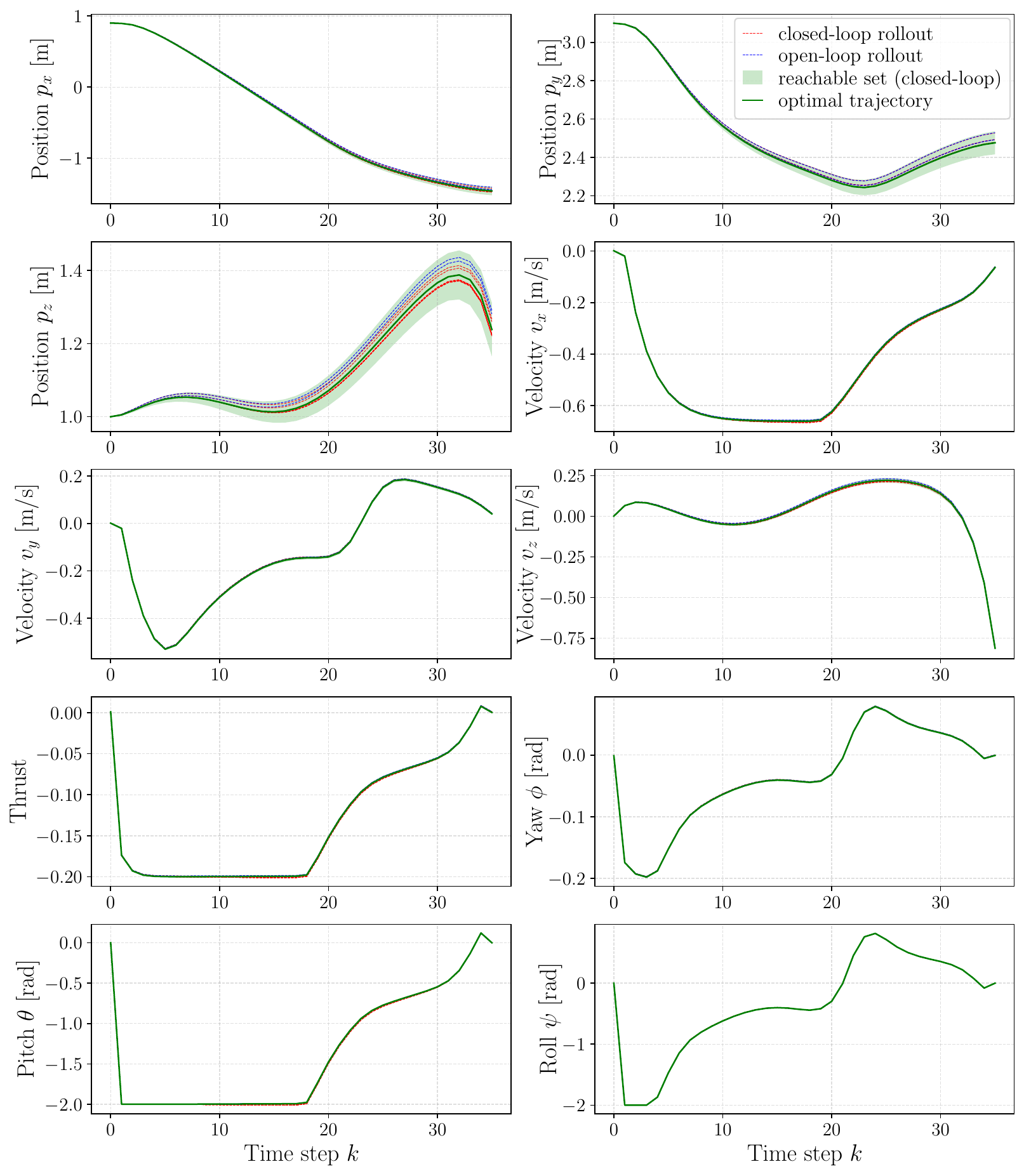}\vspace{-10pt}
    \caption{For quadrotor visuomotor control, we plot the reachable tubes (shaded green) around the planned nominal trajectory (green line) and verify that the closed-loop rollouts (blue) stay within the tubes. Note that our closed-loop rollouts experience smaller deviation from the nominal trajectory compared to the open-loop rollouts, showcasing the utility of our closed-loop controller.
    }\vspace{-15pt}
    \label{fig:quadrotor_tubes}
\end{figure}

\subsubsection{Parking-Lot Perception Error Calibration}
\label{perception_calibration_appendix}

To quantify how well the learned measurement-error bound $\hat e(x)$ captures the true perception residuals, 
we evaluate its empirical coverage on a held-out dataset. 
For each sample $\{x_i, y_i\}$, we compute the residual 
$r_i = \|p(\phi(y_i)) - h^r(x_i)\|_\infty$ 
and report the fraction of points whose residual lies below the predicted bound,
$    \mathrm{Coverage}
  = \tfrac{1}{N}\sum_{i=1}^N \mathbf{1}\{\,r_i \le \hat e(x_i)\}.
$
Empirically, we find that this recalibration is data-efficient: in the parking-lot environment, high empirical coverage is already obtained with a modest number of calibration points (App.~\ref{perception_calibration_appendix}).
Coverage is empirically evaluated on 500 randomly-sampled data points over the state space; detailed empirical coverage rates are presented in Table \ref{table:3}.

\subsubsection{Unitree G1 Humanoid}\label{app:humanoid}
Refer to the URDF of G1 with 23 DOF for joint angle bounds. The pose ($[p_x,p_y,p_z,q_x,q_y,q_z,q_w]$) bounds are $[-10,-10,-10,-1,-1,-1,-1]\times[10,10,10,1,1,1,1]$, linear velocity bounds are all $-[30,30]$, angular velocity bounds are all $[-40\pi, 40\pi]$, and joint velocity bounds are all $[-20\pi,20\pi]$. We constrain the wrist states to be within $[-0.5, 0.5]$. The input (torque) bounds are all $[-500,500]$. 
We consider initial condition and dynamics uncertainty of $\Xi=E=0.05I_{59}$, and perception uncertainty is described by $F=0.001I_{30}$. 
The planning horizon is $T = 60$ and the continuous-time dynamics are time-discretized with timestep $0.01$s. The overall computation time is $11542.795$s, with 20 SCP iterations. We note that these times are dominated by the QP solve step (9445.333s) for performing the trajectory optimization, rather than in the Riccati steps needed to compute the reachable sets; as such, these solve times can be decreased by leveraging, e.g., structure-exploiting optimal control QP solvers \cite{verschueren2022acados} or GPU acceleration \cite{schubiger2020gpu}. For simplicity, for dynamics modeling, we only consider the phase when the robot is not in contact with ground.

In this experiment, we use a slightly different trajectory cost (\ref{eq:full_trajectory_cost_humanoid}) where $x^g$ is a reference trajectory. The penalty terms in cost function are 
\begin{align} \notag
    \bar Q = & \text{diag}([0, 100, 0, 10, 100, 10, 100, \underbrace{75, \dots, 75}_{23}, 0, 100, 0, 100, \\ \notag
    & \hspace{9mm} 100, 100, \underbrace{0.5, \dots, 0.5}_{23}]), \\ \notag
    \bar P = & \text{diag}([0, 100, 50, \underbrace{100, \dots, 100}_{27}, 0, 100, 50, 100, 100, 100, \\ \notag 
    & \hspace{9mm} \underbrace{0, \dots, 0}_{23}]), 
\end{align}
$\bar R=0.001I_{23}, Q=P=10^4I_{59}, R=10^4I_{23}$.

\begin{align} \label{eq:full_trajectory_cost_humanoid}
    J_{\textrm{traj}}(\bm{z}, \bm{v}, \bm{\tau}) = & \sum_{k=0}^{T-1} ( (z_k-x^g_k)^\top \bar Q (z_k-x^g_k) + v_k^\top \bar R v_k) \\ \notag
    & + (z_T-x^g_T)^\top \bar P (z_T-x^g_T) 
\end{align}

Figs.~\ref{fig:humanoid_all_tubes_cl} and~\ref{fig:humanoid_all_tubes_ol_cl} additionally report the full trajectories and reachable tubes for all states and inputs under Algorithm~1. Fig.~\ref{fig:humanoid_all_tubes_cl} only displays the closed-loop tubes for clarity, as some of the open-loop tubes exit the tubes with a large margin and makes visualization difficult; both open-loop and closed-loop are overlaid in Fig. \ref{fig:humanoid_all_tubes_ol_cl}.
\begin{figure*}
    \centering
    \includegraphics[width=0.98\linewidth]{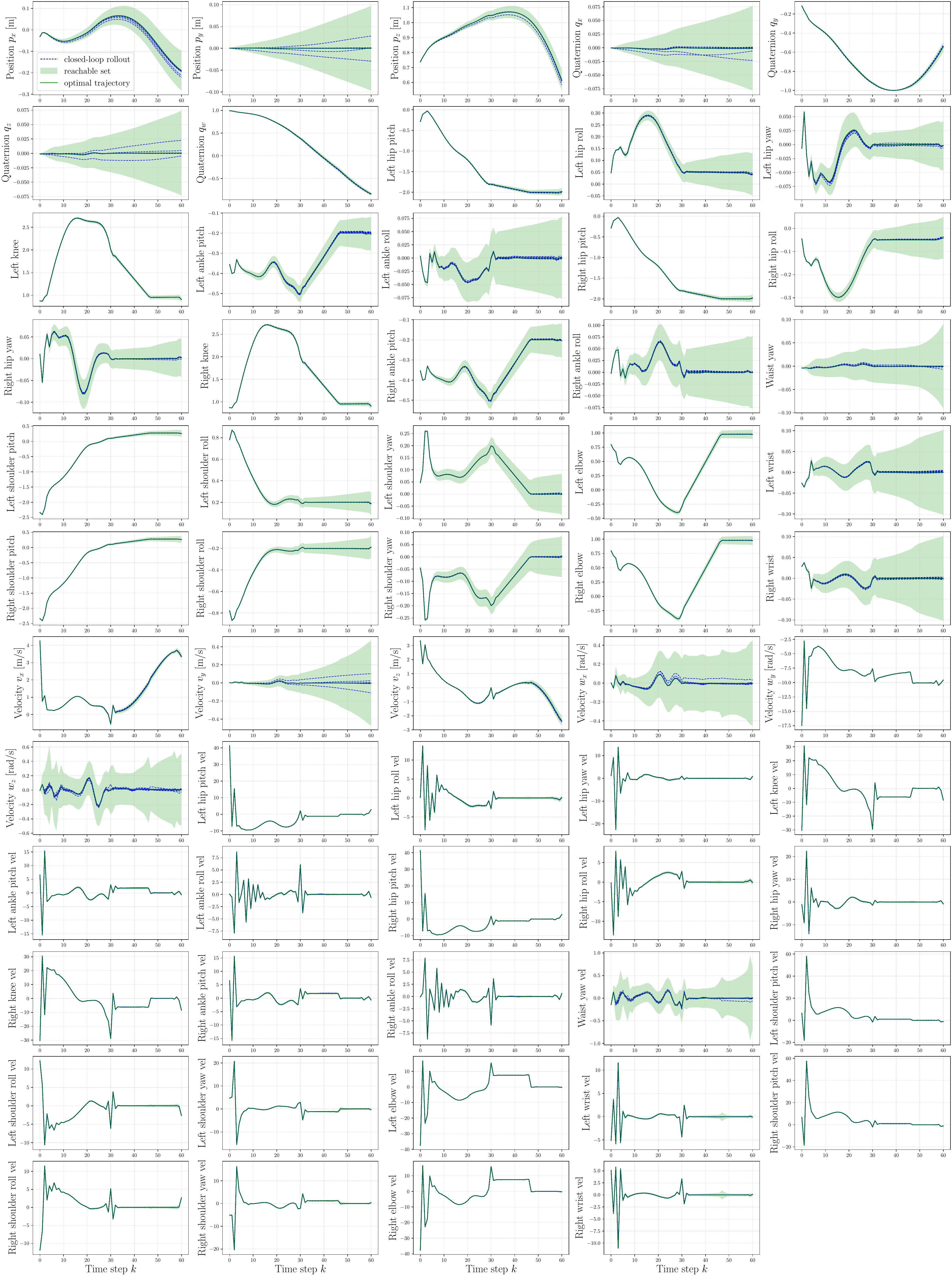}\vspace{-10pt}
    \caption{Humanoid (all tubes, closed-loop only).}
    \label{fig:humanoid_all_tubes_cl}
\end{figure*}

\begin{figure*}
    \centering
    \includegraphics[width=0.98\linewidth]{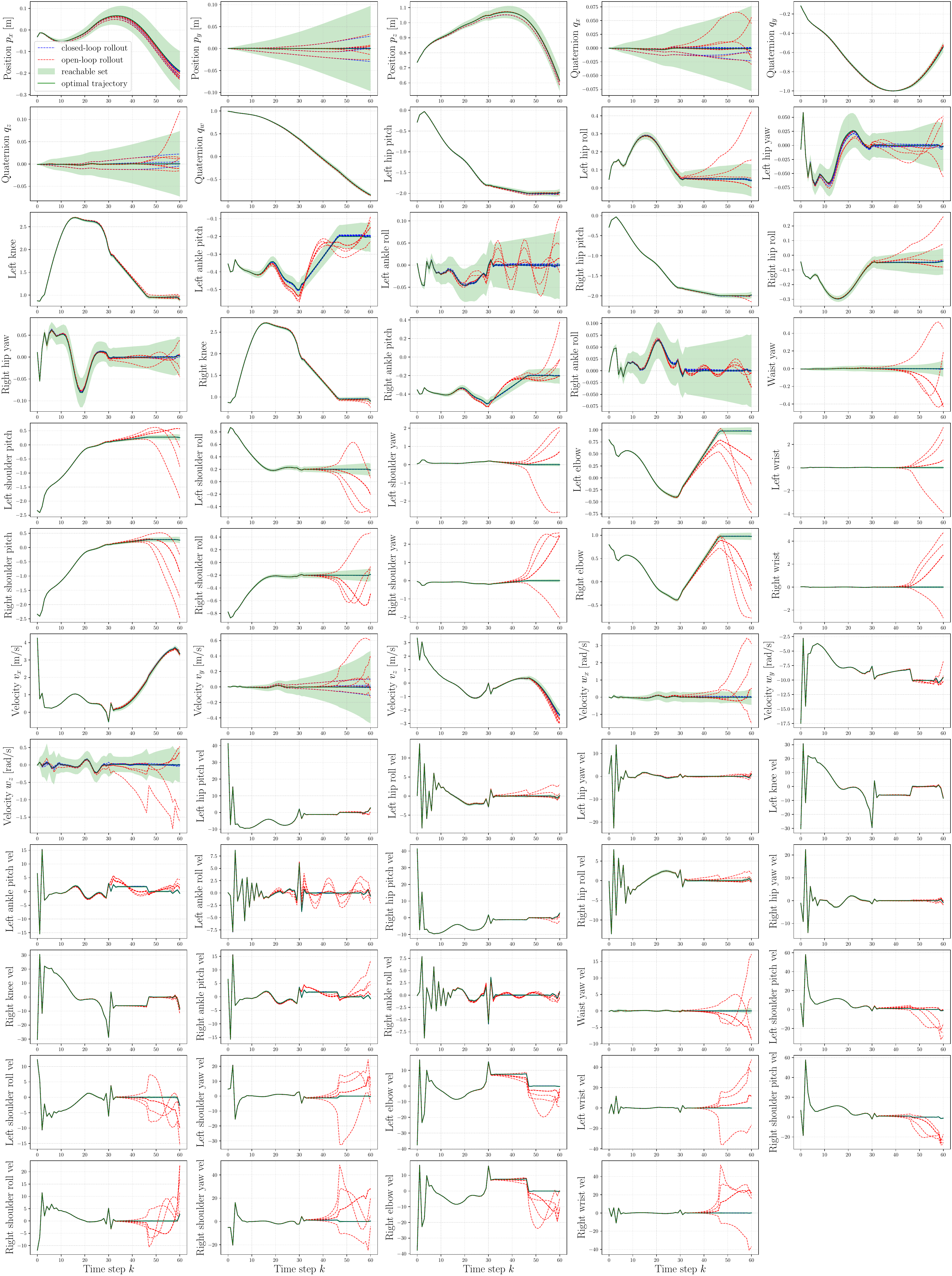}\vspace{-10pt}
    \caption{Humanoid (all tubes, open-loop and closed-loop).
}
    \label{fig:humanoid_all_tubes_ol_cl}
\end{figure*}

\subsubsection{Hardware}\label{app:hardware}

We consider a modified Dubin's car dynamics:
\begin{equation}\label{eq:dynamics_hardware}
    {\begin{bmatrix}
        \dot p_x \\ \dot p_y \\ \dot \theta \\ \dot v \\ \dot \omega
    \end{bmatrix}} = \begin{bmatrix}
        v\cos\theta \\ v\sin\theta \\ \omega \\ u_1 \\ u_2
    \end{bmatrix}
\end{equation}

The Turtlebot robot should satisfy state constraints $[p_x, p_y, p_\theta, v, \omega]\in [-5,5]\times[-5,5]\times[-2\pi,2\pi]\times[-0.2,0.4]\times[-1,1]$ (position, heading, linear velocity, angular velocity), input constraints $[u_1, u_2]\in [-0.1,0.1]\times[-0.2,0.2]$ (linear acceleration, angular acceleration). 

For this experiment, the robot plans over a horizon of $T=60$ from $\bar x_0 = (0,0,0,0,0)$ to a target state $x^g = (3,1.5,0,0,0)$.
The initial condition uncertainty is $\Xi=\textrm{diag}([0.025, 0.025, 0.05, 0.0001, 0.0001])$ and $E=\textrm{diag}([0.01, 0.01, 0.01, 0.0001, 0.0001])$.
The penalty terms in cost function are $\bar Q=\text{diag}([1,1,0,0.1,0.1]), \bar R=0.2I_2, \bar P=\text{diag}([10, 10, 1, 1, 1]), Q=P=5\times 10^6I_5, R=5\times 10^6I_2$.
Planning requires $5.779$s over 20 SCP iterations, and commands are executed at $5$~Hz via ROS~2. 
The mapping $p(\cdot)$ trained from $1096$ image–state pairs collected under varying lighting conditions is a four-layer MLP with GELU as activation functions. The hidden dimensions are 512. 
The observation matrix $C^r$ learned is 
\begin{align} \notag
    \begin{bmatrix}
        -0.184 & 0.008 & -0.983 & 0 & 0\\
        -0.435 & 0.443 & 0.784 & 0 & 0\\
         0.610 & -0.193 & -0.768 & 0 & 0
    \end{bmatrix}.
\end{align}
We use a constant perception error bound $F=0.05I_3$ achieving $99.8\%$ empirical coverage along the trajectory. 

\end{document}